\documentclass{elsarticle}
\usepackage[utf8]{inputenc}
\usepackage[english]{babel}
\usepackage{amsmath,amsfonts,amssymb,mathrsfs,amsgen,amsbsy}
\usepackage{amsthm}
\usepackage{amssymb}
\usepackage{longtable}
\usepackage{graphicx}

\usepackage{color}

\biboptions{numbers,sort&compress}

\journal{International Journal of Approximate Reasoning}

\newtheoremstyle{examplecon}
{\topsep} {\topsep}%
{\upshape}
{}
{\bfseries\scshape}
{.}
{ }
{\thmname{#1} \thmnumber{ #2}\thmnote{#3}\normalfont\enspace(continued)}

\theoremstyle{examplecon}
\newtheorem*{examplecont}{Example}

\theoremstyle{definition}
\newtheorem{definition}{Definition}
\newtheorem{lemma}{Lemma}
\newtheorem{proposition}{Proposition}
\newtheorem{theorem}{Theorem}
\newtheorem{corollary}{Corollary}
\newtheorem{example}{Example}
\theoremstyle{remark}

\newcommand{\ie}{i.e.}

\newcommand{\ifff}{iff}
\newcommand{\virg}[1]{``#1''}

\newcommand{\argfram}{\ensuremath{\Delta}}                          
\newcommand{\setarg}{\mathcal{A}}                           
\newcommand{\attackrel}{\rightarrow}                        

\newcommand{\carfun}[1]{\ensuremath{\mathrm{F}_{#1}}}                 

\newcommand{\unset}{\U}

\newcommand{\set}[1]{\{#1\}}
\newcommand{\gensem}{\ensuremath{\Sigma}}

\newcommand{\aPreferredExtension}{\ensuremath{\Pi}}

\newcommand{\R}[0]{\ensuremath{\mathcal{R}}}
\newcommand{\E}[0]{\ensuremath{\mathcal{E}}}

\newcommand{\A}[0]{\ensuremath{\mathcal{A}}}
\newcommand{\SUB}[0]{\ensuremath{\mathcal{S}}}

\newcommand{\U}[0]{\ensuremath{\mathcal{U}}}

\newcommand{\AF}[0]{\ensuremath{AF}}
\newcommand{\AFRA}[0]{\ensuremath{AFRA}}
\newcommand{\GAF}[0]{\AFRA}
\newcommand{\EAF}[0]{\ensuremath{EAF}}
\newcommand{\HEAF}[0]{\ensuremath{HEAF}}
\newcommand{\psEAF}[0]{\ensuremath{psEAF}}
\newcommand{\VAF}[0]{\ensuremath{VAF}}

\newcommand{\EAFplus}[0]{\ensuremath{EAF+}}

\newcommand{\element}[1]{\ensuremath{\mathscr{#1}}}

\newcommand{\tup}[2]{\langle #1, #2\rangle}
\newcommand{\tupthree}[3]{\langle #1, #2, #3\rangle}
\newcommand{\tupN}[1]{\langle #1\rangle}

\newcommand{\unsetAFRA}{\SUB}
\newcommand{\ASAFRA}[1]{\ensuremath{\mathcal{AS}_{#1}}}
\newcommand{\arga}{\ensuremath{A}}

\newcommand{\argb}{\ensuremath{B}}
\newcommand{\argc}{\ensuremath{C}}

\newcommand{\argd}{\ensuremath{D}}
\newcommand{\arge}{\ensuremath{E}}
\newcommand{\argf}{\ensuremath{F}}
\newcommand{\argg}{\ensuremath{G}}

\newcommand{\argn}{\ensuremath{N}}

\newcommand{\argp}{\ensuremath{P}}

\newcommand{\argx}{\ensuremath{X}}
\newcommand{\argxprime}{\ensuremath{X'}}
\newcommand{\argy}{\ensuremath{Y}}
\newcommand{\argyprime}{\ensuremath{Y'}}
\newcommand{\argz}{\ensuremath{Z}}

\newcommand{\elementa}{\ensuremath{\element{A}}}

\newcommand{\elementb}{\ensuremath{\element{B}}}

\newcommand{\elementc}{\ensuremath{\element{C}}}

\newcommand{\elementu}{\ensuremath{\element{U}}}
\newcommand{\elementv}{\ensuremath{\element{V}}}
\newcommand{\elementw}{\ensuremath{\element{W}}}
\newcommand{\elementx}{\ensuremath{\element{X}}}

\newcommand{\attackalpha}{\ensuremath{\alpha}}
\newcommand{\attackbeta}{\ensuremath{\beta}}
\newcommand{\attackgamma}{\ensuremath{\gamma}}
\newcommand{\attackdelta}{\ensuremath{\delta}}
\newcommand{\attackepsilon}{\ensuremath{\epsilon}}
\newcommand{\attackzeta}{\ensuremath{\zeta}}
\newcommand{\attacketa}{\ensuremath{\eta}}
\newcommand{\attacktheta}{\ensuremath{\theta}}
\newcommand{\attackiota}{\ensuremath{\iota}}
\newcommand{\attackkappa}{\ensuremath{\kappa}}

\newcommand{\attacksAF}[2]{\ensuremath{(#1,#2) \in \attackrel}}
\newcommand{\src}[1]{\ensuremath{src(#1)}}
\newcommand{\trg}[1]{\ensuremath{trg(#1)}}

\newcommand{\AFRattacks}[2]{\ensuremath{#1 \rightarrow_{R} #2}}

\newcommand{\toAF}[1]{#1_{AF}}
\newcommand{\toAFRA}[1]{#1_R}

\newcommand{\Atilde}{\ensuremath{\widetilde{\mathcal{A}}}}
\newcommand{\Rtilde}{\ensuremath{\widetilde{\mathcal{R}}}}

\newcommand{\anAFRA}{\ensuremath{\Gamma}}
\newcommand{\bobsAFRA}{\ensuremath{\Gamma_{Bob}}}
\newcommand{\anotherAFRA}{\ensuremath{\widehat{\Gamma}}}
\newcommand{\athirdAFRA}{\ensuremath{\overline{\Gamma}}}

\newcommand{\charfunction}[1]{\ensuremath{\mathbb{F}_{#1}}}

\newcommand{\ModgilArgs}{\ensuremath{Args}}
\newcommand{\ModgilArgsCF}{\ensuremath{ArgsC}}
\newcommand{\ModgilR}{\ensuremath{R}}
\newcommand{\ModgilD}{\ensuremath{D}}

\newcommand{\ModgilDefeatSWord}{\ensuremath{defeat_S}}
\newcommand{\ModgilDefeatS}[3]{\ensuremath{#1 \rightarrow^{#3} #2}}
\newcommand{\ModgilNDefeatS}[3]{\ensuremath{#1 \nrightarrow^{#3} #2}}
\newcommand{\ModgilEAFTuple}{\ensuremath{\tupthree{\ModgilArgs}{\ModgilR}{\ModgilD}}}
\newcommand{\aModgilEAF}{\ensuremath{\Lambda}}
\newcommand{\ModgilRS}{\ensuremath{R_S}}
\newcommand{\ModgilCharFunct}{\ensuremath{F_{\aModgilEAF}}}

\newcommand{\unsetModgil}{\ensuremath{S}}

\newcommand{\BvVArgsC}{\ensuremath{\mathcal{A}_{C}}}
\newcommand{\BvVArgsNot}{\ensuremath{\overline{\mathcal{A}}}}
\newcommand{\BvVnot}{\ensuremath{not}}
\newcommand{\BvVArgsDiesis}{\ensuremath{\mathcal{A}_{\#}}}
\newcommand{\BvVDiesis}{\ensuremath{\#}}
\newcommand{\BvVtuple}{\ensuremath{\tupN{\BvVArgsC, \BvVArgsNot, \BvVnot, \BvVArgsDiesis, \BvVDiesis}}}
\newcommand{\BvVarga}{\ensuremath{a}}
\newcommand{\BvVargb}{\ensuremath{b}}
\newcommand{\BvVargc}{\ensuremath{c}}
\newcommand{\BvVargd}{\ensuremath{d}}

\newcommand{\aBvV}{\ensuremath{\Upsilon}}

\newcommand{\dungconffree}{D-conflict-free}				
\newcommand{\dungacceptable}{D-acceptable}
\newcommand{\dungadmissible}{D-admissible}
\newcommand{\dungpreferred}{D-preferred}
\newcommand{\dungstable}{D-stable}
\newcommand{\dungcomplete}{D-complete}
\newcommand{\dunggrounded}{D-grounded}
\newcommand{\dungsemistable}{D-semi-stable}
\newcommand{\dungrange}{D-range}
\newcommand{\dungideal}{D-ideal}
\newcommand{\dungcharacteristic}{D-characteristic}
\newcommand{\dungplus}[1]{{#1}^{D+}}								

\newcommand{\exttoafra}[1]{{#1}^{\rightarrow \AFRA}} 
\newcommand{\opexttoafra}{\ensuremath{\rightarrow \AFRA}}  
\newcommand{\drange}[1]{\ensuremath{\mathit{Drange}(#1)}}
\newcommand{\range}[1]{\ensuremath{range(#1)}}

\begin{document}

\begin{frontmatter}

\title{\AFRA: Argumentation Framework with Recursive Attacks}
\author[unibs]{Pietro Baroni}
\ead{pietro.baroni@ing.unibs.it}

\author[unibs]{Federico Cerutti\corref{corr}}
\ead{federico.cerutti@ing.unibs.it}

\author[unibs]{Massimiliano Giacomin}
\ead{massimiliano.giacomin@ing.unibs.it}

\author[unibs]{Giovanni Guida}
\ead{giovanni.guida@ing.unibs.it}

\cortext[corr]{Corresponding author}

\address[unibs]{Dip. di Elettronica per l'Automazione, University of Brescia, Via Branze 38, I-25123 Brescia, Italy}

\begin{abstract}
The issue of representing attacks to attacks in argumentation is receiving an increasing attention as a useful conceptual modelling tool in several contexts. In this paper we present \AFRA, a formalism encompassing unlimited recursive attacks within argumentation frameworks.
\AFRA{} satisfies the basic requirements of definition simplicity and rigorous compatibility with Dung's theory of argumentation. This paper provides a complete development of the \AFRA{} formalism complemented by illustrative examples and a detailed comparison with other recursive attack formalizations.
\end{abstract}


\end{frontmatter}

\section{Introduction}

An argumentation framework ($\AF$ in the following), as introduced in the seminal paper by Dung \cite{dung1995}, is an abstract structure consisting of a set of elements, called \emph{arguments}, whose origin, nature and possible internal organization is not specified, and by a binary relation of \emph{attack} on the set of arguments, whose meaning is not specified either.
This abstract formalism has been shown to encompass a large variety of more specific formalisms in areas ranging from nonmonotonic reasoning to logic programming and game theory, and, as such, is widely regarded as a powerful tool for theoretical analysis.
Several variations of the original $\AF$ formalism have been proposed in the literature.
On one hand, some approaches enrich the original framework with additional concepts, necessary to modelling in a \virg{natural} way specific reasoning situations.
This is the case, for instance, of preference-based argumentation \cite{amgoud&cayrol2002,amgoud&kaci2007}, where a preference ordering among arguments is considered, of value-based argumentation \cite{benchcapon2003}, where a value is associated to arguments in order to account for the concept of preference (an investigation on the relations between preference-based and valued-based argumentation is given in \cite{kaci&vandertorre2007}), of bipolar argumentation \cite{amgoudetal2008,cayroletal2010}, where a relation of support between arguments is considered besides that of attack, or of weighted argument systems \cite{dunneetal2009}, where a weight indicates the relative strength of attacks.
On the other hand, some proposals investigate generalized versions of the original $\AF$ definition (in particular, of the notion of attack), without introducing any additional concept within the basic scheme, as in \cite{modgil2007,benchcapon&modgil2008,modgil2009}.
This paper lies in the latter line of investigation and pursues the goal of generalizing the $\AF$ notion of attack by allowing an attack, starting from an argument, to be directed not just towards an argument but also towards any other attack. This will be achieved by a recursive definition of the attack, that leads to the proposal of a new framework called $\AFRA$ (Argumentation Framework with Recursive Attacks). 

The paper is organized as follows.
Section \ref{sec_background} recalls the basic notions and fundamental properties of Dung's argumentation framework.
Section \ref{sec_motivations} introduces the definition of \AFRA{} accompanied by a discussion of its motivations and objectives. In particular, \AFRA{} is required to parallel the semantics notions of Dung's theory and their fundamental properties while extending them to recursive attacks in an intuitively plausible and formally simple way.
Section \ref{sec_semantics_afra_1} introduces the generalized version of the basic notions of defeat, conflict-free set, acceptable argument, characteristic function, and admissible set showing that the relevant requirements stated in Section \ref{sec_motivations} hold.
Section \ref{sec_semantics_afra_2} extends to \AFRA{} the definitions of complete, grounded, preferred, stable, semi-stable and ideal semantics, showing that their required properties, paralleling the traditional ones, hold.
Sections \ref{sec_compatibility_afra_af} and \ref{SecGAFAF} deal with further relationships between \AFRA{} and \AF. The former shows that when an \AFRA{} coincides with an \AF{} (since no attacks to attacks are present) all the generalized \AFRA{} notions are fully compatible with the original ones. In the latter a method to express an \AFRA{} as an \AF{} is provided.
Section \ref{sec_discuss} draws a detailed comparison of \AFRA{} with the related formalisms Extended Argumentation Framework (\EAF) and Higher Order Argumentation Framework (HOAF). Finally Section \ref{sec_concl} summarizes the main contributions of the paper and discusses directions for future research.

\section{Background notions}
\label{sec_background}

In Dung's theory an argumentation framework (\AF) is a pair $\tup{\setarg}{\attackrel}$ where $\setarg$ is a set of arguments and $\attackrel \subseteq \setarg \times \setarg$ is a binary relation on it. 
The terse intuition behind this formalism is that arguments may attack each other 
and useful formal definitions and theoretical investigations may be built on this simple basis. 
In particular, the notions recalled in Definition \ref{defrecall} lie at the heart 
of the definitions of Dung's \emph{argumentation semantics}\footnote{The letter \virg{D} prefixed to the terms introduced in Definitions \ref{defrecall} and \ref{def_sem_recall} denotes that they specifically refer to the Dung's proposal.}, 
each of them representing a formal way of determining the conflict outcome \cite{baroni&giacomin2009}.

\begin{definition}
\label{defrecall}
Given an \AF{} $\argfram = \tup{\setarg}{\attackrel}$:
\begin{itemize}
  \item a set $\unset \subseteq \setarg$ is \emph{\dungconffree} 
        if $\nexists \arga, \argb \in \unset$ s.t. $\attacksAF{\arga}{\argb}$;
  \item an argument $\arga \in \setarg$ is \emph{\dungacceptable} with respect to a set $\unset \subseteq \setarg$
        (or, equivalently, is defended by $\unset$) 
        if $\forall \argb \in \setarg$ s.t. $\attacksAF{\argb}{\arga}$, 
        $\exists \argc \in \unset$ s.t. $\attacksAF{\argc}{\argb}$;
  \item the function $\carfun{\argfram}: 2^{\setarg} \rightarrow 2^{\setarg}$ such that
        $\carfun{\argfram}(\U) = \set{\arga \mid \arga \mbox{ is } \mbox{\dungacceptable} \mbox{ w.r.t. } \U}$
        is called the \emph{\dungcharacteristic{} function} of \argfram;
  \item a set $\unset \subseteq \setarg$ is \emph{\dungadmissible} 
        if $\unset$ is \dungconffree{} and every element of $\unset$ is \dungacceptable{} with respect to $\unset$,
        \ie{} $\unset \subseteq \carfun{\argfram}(\unset)$.
\end{itemize}
\end{definition}

An argumentation semantics identifies for any argumentation framework a set of \emph{extensions}, 
namely sets of arguments which are \virg{collectively acceptable}, 
or, in other words, able to survive together the conflict represented by the attack relation: for instance
arguments that belong to all of extensions can be considered \emph{skeptically justified}, while arguments belonging to at least an extension can be considered \emph{credulously justified}.
We recall that while the \emph{grounded} and \emph{ideal} semantics always identify a unique extension for a given argumentation framework 
(called grounded and ideal extension, respectively), 
the \emph{preferred}, \emph{stable}, and \emph{semi-stable} semantics can identify several extensions
(called preferred, stable, and semi-stable extensions, respectively)\footnote{The reader is referred 
to \cite{AIJ07,BaroniGiacominijar09} for a detailed discussion and a comparison 
concerning the different semantics proposed in the literature.}.

\begin{definition}
\label{def_sem_recall}
Given an \AF{} $\argfram = \tup{\setarg}{\attackrel}$:
\begin{itemize}
  \item a set $\unset \subseteq \setarg$ is a \emph{\dungcomplete{} extension} 
        if $\unset$ is \dungadmissible{} 
        and $\forall \arga \in \setarg$ s.t. $\arga$ is \dungacceptable{} w.r.t. $\unset$,  $\arga \in \unset$;
  \item a set $\unset \subseteq \setarg$ is the \emph{\dunggrounded{} extension} 
        if $\unset$ is the least (w.r.t. set inclusion) fixed point\footnote{It is shown in Theorem 25 of \cite{dung1995} that the grounded extension can be equivalently characterized as the least \dungcomplete{} extension.} of the \dungcharacteristic{} function $\carfun{\argfram}$;        
  \item a set $\unset \subseteq \setarg$ is a \emph{\dungpreferred{} extension} 
        if $\unset$ is a maximal (w.r.t. set inclusion) \dungadmissible{} set;
  \item a set $\unset \subseteq \setarg$ is a \emph{\dungstable{} extension} 
        if $\unset$ is \dungconffree{} and $\forall \arga \in \setarg \setminus \unset$, 
        $\exists \argb \in \unset$ s.t. $\attacksAF{\argb}{\arga}$;
  \item a set $\unset \subseteq \setarg$ is a \emph{\dungsemistable{} extension} 
        if $\unset$ is a \dungcomplete{} extension with maximal (w.r.t. set inclusion) \dungrange{} 
        (given a set $\unset \subseteq \setarg$, the \dungrange{} of $\unset$, denoted as $\drange{\unset}$, 
         is $\unset \cup \dungplus{\unset}$ 
         where $\dungplus{\unset} = \set{\arga \in \setarg | \exists \argb \in \unset$ 
         s.t. $ \attacksAF{\argb}{\arga} }$);
  \item a set $\unset \subseteq \setarg$ is the \emph{\dungideal{} extension} 
        if $\unset$ is the maximal (w.r.t. set inclusion) \dungideal{} set 
        (a set $\unset \subseteq \setarg$ is \emph{\dungideal{}} if it is \dungadmissible{} 
        and $\forall \E$ s.t. $\E$ is a \dungpreferred{} extension, 
        $\unset \subseteq \E$).
\end{itemize}
\end{definition}

It is easy to note that any extension of the above semantics is  a \dungadmissible{} set,
that is, it is able to defend all of its arguments 
(in the sense that arguments are \dungacceptable{} w.r.t. the extension itself).
The notion of acceptability and the related notion of admissibility are supported by intuition 
and satisfy a set of fundamental properties which in turn entail several desirable consequences, holding even in the infinite case, 
such as the existence of preferred extensions as well as the existence and uniqueness of the grounded extension.
These properties are recalled in the following proposition \cite{dung1995}.

\begin{proposition}
\label{prop_requirements}
  Given an \AF{} $\argfram = \tup{\setarg}{\attackrel}$:
  \begin{itemize}
    \item the \dungcharacteristic{} function preserves \dungconffree ness, 
          \ie{} given a set $\unset \subseteq \setarg$,
          if $\unset$ is \dungconffree{} then also $\carfun{\argfram}(\unset)$ is \dungconffree;
    \item the \dungcharacteristic{} function is monotonic,
          \ie{} if $\unset_1 \subseteq \unset_2$,
          then $\carfun{\argfram}(\unset_1) \subseteq \carfun{\argfram}(\unset_2)$;
    \item Dung's fundamental lemma: given a \dungadmissible{} set $\unset \subseteq \setarg$
          and two arguments $\arga, \arga' \in \setarg$ that are \dungacceptable{} w.r.t. $\unset$,
          it holds that
          $\unset' = \unset \cup \set{\arga}$ is \dungadmissible{}
          and $\arga'$ is \dungacceptable{} w.r.t. $\unset'$;
    \item the set of all admissible sets form a complete partial order w.r.t. set inclusion.
  \end{itemize}
\end{proposition}

\section{Motivations and requirements}
\label{sec_motivations}

In Dung's theory, arguments are regarded as the only entities that may be in conflict with each other and may be defeasible.
The issue of extending the framework in such a way as also attacks are allowed to feature these properties has recently received significant attention in the literature.
In fact, enabling attacks to attacks and considering them defeasible turns out to provide a useful and intuitively plausible formal counterpart to representation and reasoning patterns commonly adopted in various contexts.
For instance, an approach to reasoning about preferences based on attacks to attacks has been introduced in \cite{modgil2007} and extensively developed in \cite{modgil2009}.
In \cite{boellaetal2008,boellaetal2008b} attacks to attacks are considered in the context of reasoning about coalitions, while in \cite{barringeretal2005} attacks to attacks are discussed in connection with the notions of strength, support and temporal dynamics.

The present paper contributes to the research line on formalizing attacks to attacks in argumentation by pursuing the following main objectives:

\begin{itemize}
\item encompassing an unrestricted recursive notion of attack to attack;
\item keeping the proposed formalism as simple as possible;
\item encompassing Dung's \AF{} as a special case of the proposed formalism;
\item ensuring compatibility between the semantics notions in the proposed formalism and those in Dung's \AF.
\end{itemize}

As to the first point, in some previous proposals (e.g. \cite{modgil2007,modgil2009}) only one level of recursion is allowed, i.e. attacks attacking other attacks can not in turn be attacked.
While this choice may be justified in specific contexts (e.g. reasoning about preferences), we aim at proposing a more general formalism which is able to accommodate various kinds of representation and reasoning needs related to recursive attacks.
In particular, further levels of recursive attacks can be considered in the area of modelling decision processes as shown by the following example, which will be used throughout the paper to illustrate the main concepts of the proposed approach.

Suppose Bob is deciding about his Christmas holidays and, as a general rule of thumb, he is willing to buy cheap last minute offers. Suppose two such offers are available, one for a week in Gstaad and another for a week in Cuba.
Then, using his behavioral rule, Bob can build two arguments, one, let say $G$, whose premise is \virg{There is a last minute offer for Gstaad} and whose conclusion is \virg{I should go to Gstaad}, the other, let say $C$,  whose premise is \virg{There is a last minute offer for Cuba} and whose conclusion is \virg{I should go to Cuba}.
As the two choices are incompatible, $G$ and $C$ attack each other, a situation giving rise to an undetermined choice. 
Suppose however that Bob has a preference $P$ for skiing and knows that Gstaad is a renowned ski resort. The point now is: how can we represent this preference?
$P$ might be represented implicitly by suppressing the attack from $C$ to $G$, but this is unsatisfactory, since it would prevent further reasoning on $P$, as described below.
So let us consider $P$ as an argument whose premise is \virg{Bob likes skiing} and whose conclusion is \virg{If possible, Bob prefers a ski resort}.
$P$ might then attack $C$, but this would not be sound since $P$ is not actually in contrast with the existence of a good last minute offer for Cuba and the fact that, according to Bob's general behavioral rule, this provides him with a good reason for going to Cuba. 
Thus, following \cite{modgil2009}, it seems more reasonable to represent $P$ as attacking the attack from $C$ to $G$, causing $G$ to prevail. Note that the attack from $C$ to $G$ is not suppressed, but only made ineffective, in the specific situation at hand, due to the attack of $P$.

Assume now that Bob learns that there have been no snowfalls in Gstaad since one month and from this fact he derives that it might not be possible to ski there. This argument ($N$), whose premise is \virg{The weather report informs that in Gstaad there were no snowfalls since one month} and whose conclusion is \virg{It is not possible to ski in Gstaad}, does not affect neither the existence of last minute offers for Gstaad nor Bob's general preference for ski, rather it affects the ability of this preference to affect the choice between Gstaad and Cuba. 
Thus argument $N$ attacks the attack originated from $P$.  

Suppose finally that eventually Bob is informed that in Gstaad it is anyway possible to ski, thanks to a good amount of artificial snow. This leads to building an argument, let say $A$, which attacks $N$, thus in turn reinstating the attack originated from $P$ and intuitively supporting the choice of Gstaad.  A graphical illustration of this example is provided in Figure \ref{fig_esempio}.

\begin{figure}[ht]
	\centering
	\includegraphics[scale=0.25]{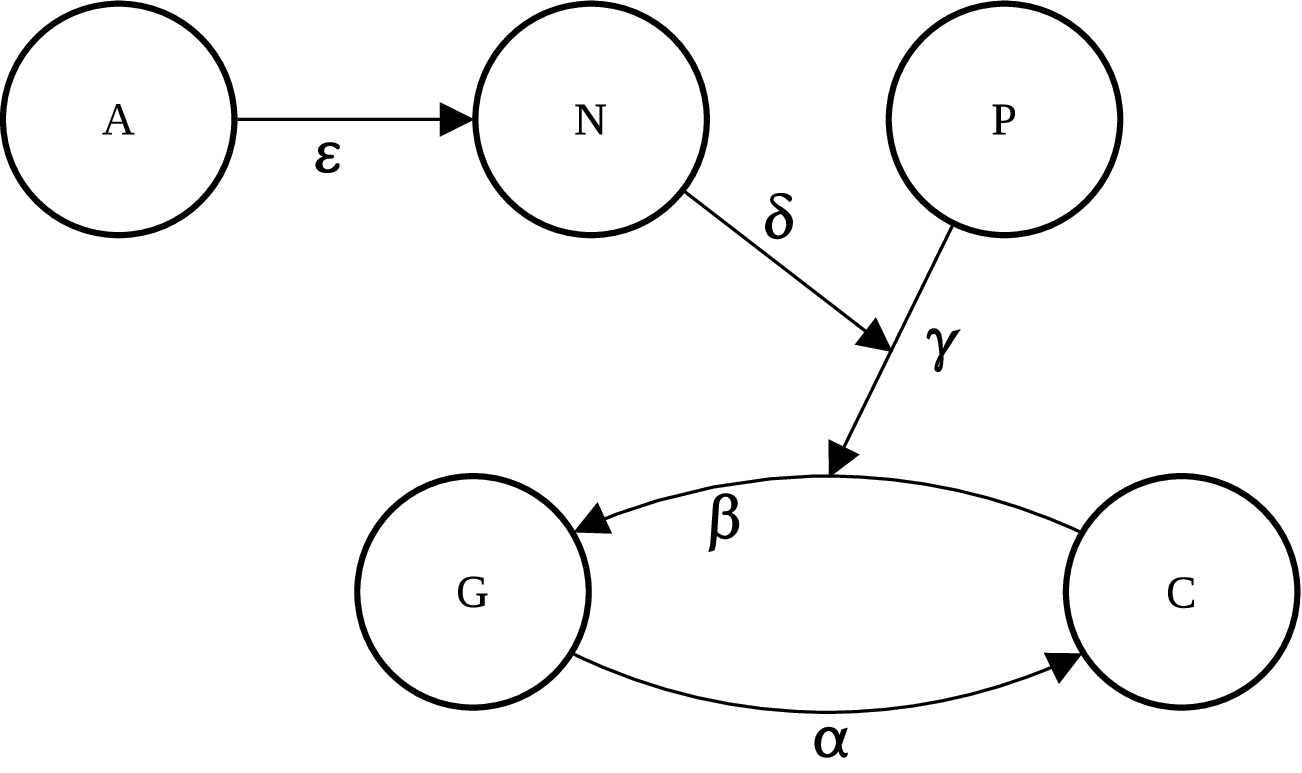}
	\caption{Bob's last minute dilemma.}
	\label{fig_esempio}
\end{figure}

As pointed out by one of the reviewers of this paper, alternative formalizations of this example not involving attacks to attacks are possible.
For instance, from the general preference for skiing, represented in the example by argument $P$, one might derive a distinct and more specific argument $P'$ representing the preference for Gstaad over Cuba. In this case argument $P'$ (instead of $P$) would attack $\beta$ through $\gamma$ and argument $N$ would attack $P'$ (instead of $\gamma$) through $\delta$. 
Of course, the representation adopted for this example -- like any formal representation of a real situation -- is a matter of modelling choice.
In general, we do not claim that there are indisputable, theoretical reasons for asserting that recursive attacks are strictly necessary.
Indeed, technically speaking, extended argumentation frameworks encompassing attacks to attacks do not feature an augmented expressive power with respect to Dung's formalism, as they can be translated into traditional argumentation frameworks, as shown for instance in \cite{modgil2009} and in Section \ref{SecGAFAF} of the present paper.
From a modelling point of view, however, it can be observed that attacks to attacks offer a useful tool supporting a natural representation of some reasoning patterns.

\begin{figure}[ht]
	\centering
	\includegraphics{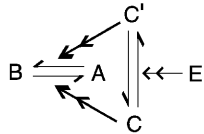}
	\caption{The \EAF{} for the weather forecast example (Fig. 3 in \cite{modgil2009}).}
	\label{fig_esempio_modgil}
\end{figure}

As a further example, consider the case presented in \cite{modgil2009} concerning two agents P and Q
exchanging arguments about weather forecasts (see Figure \ref{fig_esempio_modgil}).
Argument $A$, asserted by agent P, can be synthesized as  \virg{Today will be dry in London since the BBC forecast sunshine}, while agent Q asserts argument $B$ \virg{Today will be wet in London since CNN forecast rain}. Arguments $A$ and $B$ have contradictory conclusions and therefore attack each other. 
Preferences may then be expressed by P and Q in order to resolve this undecided situation.
For instance P may state an argument $C$  \virg{But the BBC are more trustworthy than CNN}, which expresses a preference for BBC, while Q may reply with an argument $C'$ \virg{However, statistically CNN are more accurate forecasters than the BBC} expressing a preference for CNN.
The two conflicting preferences attack each other and, according to the preference modeling adopted in \cite{modgil2009}, $C$ attacks the attack from $B$ to $A$, while $C'$ attacks the attack from $A$ to $B$.
Agent Q may then state an argument $E$ asserting that \virg{Basing a comparison on statistics is more rigorous and rational than basing a comparison on your instincts
 about their relative trustworthiness}. As argument $E$ expresses a preference for $C'$ over $C$, $E$ attacks the attack from $C$ to $C'$.
Now, in order to see how recursive attacks may play a role  in this context, consider the following additional argument $F$ asserted by P: \virg{However, BBC has recently changed its whether forecast model, no information on the new model is available; therefore statistics on CNN loses prevalence over personal opinion about BBC}.
$F$ does not attack neither $C'$ that states the preference for CNN's weather forecast over BBC's one based upon statistics, nor $E$, which states the general principle that basing a comparison on statistics is more rigorous and rational than basing a comparison on instincts. Obviously, it does not attack neither $C$, nor $A$, nor $B$. 
In fact, $F$ attacks the assumption that $E$ affects the attacks between $C$ and $C'$: while it is generally accepted that basing a comparison on statistics is more rigorous that basing a comparison on personal intuition, in the case at hand, existing statistics are not decisive for a comparison between the accuracy of CNN and BBC forecasts.
In other words $F$ provides a good reason for believing that $E$ does not attack the attack from $C'$ to $C$ and this can be modelled as an attack from $F$ to the attack originating from $E$. Therefore the situation remains undecided and both attacks between $A$ and $B$ are still in force.

Given the kind of representation needs illustrated above, we pursue the second and third objectives stated above by introducing in a rather straightforward way the fundamental definition of our proposal, namely the concept of argumentation framework with recursive attacks.

\begin{definition}[$\AFRA$]
\label{def_afra}
An Argumentation Framework with Recursive Attacks ($\AFRA$) is a pair $\tup{\A}{\R}$ where:
\begin{itemize}
\item $\A$ is a set of arguments;
\item $\R$ is a set of attacks, namely pairs $(\arga, \elementx)$ s.t. $\arga \in \A$ and ($\elementx \in \R$ or $\elementx \in \A$).
\end{itemize}
Given an attack $\attackalpha=(\arga, \elementx) \in \R$, we say that $\arga$ is the source of $\attackalpha$, denoted as $\src{\attackalpha}=\arga$ and $\elementx$ is the target of $\attackalpha$, denoted as $\trg{\attackalpha}=\elementx$.

When useful, we will denote an attack to attack explicitly showing all the recursive steps implied by its definition; for instance $(\arga, (\argb, \argc))$ means $(\arga, \attackalpha)$ where $\attackalpha = (\argb, \argc)$.
\end{definition}

The formalization of Bob's last minute dilemma in terms of \AFRA{} gives a simple illustration of the use of the formalism.

\begin{example}[Bob's last minute dilemma]
\label{ex_bob_afra}

Let $\bobsAFRA = \tup{\A}{\R}$ be an \AFRA{} where: $\A = \set{\argc,$ $\argg,$ $\argp,$ $\argn,$ $\arga}$ and $\R = \set{\attackalpha,$ $\attackbeta,$ $\attackgamma,$ $\attackdelta,$ $\attackepsilon}$, with 
$\attackalpha = (\argg, \argc)$,
$\attackbeta = (\argc, \argg)$,
$\attackgamma = (\argp, \attackbeta)$,
$\attackdelta = (\argn, \attackgamma)$,
$\attackepsilon = (\arga, \argn)$.
\end{example}

As to our third objective, it can be noted that an \AFRA{} is also an \AF{} when $\R$ does not include pairs $(\arga, \elementx)$ such that $\elementx \in \R$.

The fourth high-level objective of \virg{compatibility} concerns the \AFRA{} semantics notions which will be introduced in Sections \ref{sec_semantics_afra_1} and \ref{sec_semantics_afra_2}.
The underlying idea is that the basic concepts of conflict-freeness, acceptability, admissibility and the various proposals of extension-based semantics are formally introduced in the context of \AFRA, by explicitly considering both arguments and attacks.
We remark in particular that, according to Definition \ref{def_afra}, we regard attacks as entities which are rooted in arguments and, as a consequence, we require that their inclusion in an extension is possible only in case of inclusion of their source argument too.
This choice ensures preservation of the main lines of Dung's well-established conceptual framework for semantics definition, while anyway reflecting the extended (in a sense, empowered) role ascribed to attacks in \AFRA, in particular their defeasibility. 

From a more formal perspective, the objective of \virg{compatibility} leads to the following requirements:
\begin{itemize}
\item the fundamental properties listed in Proposition \ref{prop_requirements} should still hold for the parallel concepts introduced in the context of \AFRA;
\item in the case where an \AFRA{} is also an \AF, a bijective correspondence between the semantics notions according to the two formalisms should hold.  
\end{itemize}

The definition of semantics notions for \AFRA{} in accordance with the objectives discussed above is carried out in  Sections \ref{sec_semantics_afra_1} and \ref{sec_semantics_afra_2}.

\section{Basic semantic notions for \AFRA}
\label{sec_semantics_afra_1}

\subsection{Defeat and conflict-free sets}

As a starting point for the definition of any semantics-related notion we consider the concept of defeat.
According to the role played by attacks in \AFRA{} we introduce a notion of direct defeat which regards attacks, rather than their source arguments, as the subjects able to defeat arguments or other attacks. This is also coherent with the fact that an attack can be made ineffective by attacking the attack itself rather than its source.

\begin{definition}[Direct Defeat]\label{defdirdef}
Let $\tup{\A}{\R}$ be an \GAF, $\attackalpha \in \R$, $\elementv \in \A \cup \R$: 
$\attackalpha$ directly defeats $\elementv$ iff $\elementv = \trg{\attackalpha}$.
\end{definition}

Moreover, according to the idea that an attack is strictly related to its source, we introduce a notion of indirect defeat for an attack, corresponding to the situation where its source receives a direct defeat.

\begin{definition}[Indirect Defeat]
\label{defindirdef}
Let $\tup{\A}{\R}$ be an \GAF{} and $\attackalpha, \attackbeta \in \R$:
if $\attackalpha$ directly defeats $\src{\attackbeta}$ then 
$\attackalpha$ indirectly defeats $\attackbeta$.
\end{definition}

\begin{examplecont}[\ref{ex_bob_afra}]
In \bobsAFRA{} there are the following direct and indirect defeats: $\attackalpha$ directly defeats $\argc$; $\attackalpha$ indirectly defeats $\attackbeta$; $\attackbeta$ directly defeats $\argg$; $\attackbeta$ indirectly defeats $\attackalpha$; $\attackgamma$ directly defeats $\attackbeta$; $\attackdelta$ directly defeats $\attackgamma$; $\attackepsilon$ directly defeats $\argn$; $\attackepsilon$ indirectly defeats $\attackdelta$.
\end{examplecont}

As a special, but significant, situation note that in case of a self-attacking argument, exemplified by the \AFRA{} $\tup{\set{\arga}}{\set{\attackalpha}}$ with $\attackalpha=(\arga,\arga)$, $\attackalpha$ directly defeats $\arga$ and indirectly defeats itself.

Summing up, a defeat is either a direct or indirect defeat.

\begin{definition}[Defeat]
\label{defeat_gaf}
Let $\tup{\A}{\R}$ be an \GAF, $\attackalpha \in \R$, $\elementv \in \A \cup \R$: 
$\attackalpha$ defeats $\elementv$, denoted as $\AFRattacks{\attackalpha}{\elementv}$, iff $\attackalpha$ directly or indirectly defeats $\elementv$.
\end{definition}

The definition of conflict-free set follows directly, requiring the absence of defeats.

\begin{definition}[Conflict--free set]
\label{conflict-free_gaf}
Let $\tup{\A}{\R}$ be an \GAF, $\SUB \subseteq \A \cup \R$ is conflict--free iff $\nexists \elementv, \elementw \in \SUB$ s.t. $\AFRattacks{\elementv}{\elementw}$.
\end{definition}

The definition of conflict-free set for \AFRA{} is formally quite similar to the corresponding one in \AF{} but they feature substantial differences.
A first one, which is quite evident and common to other \AFRA{} notions, concerns the fact that a set of arguments and attacks, rather than just a set of arguments is considered.
A slightly subtler one, related to the underlying notion of defeat, consists in the fact that in \AFRA{} every set of arguments $\unset \subseteq \setarg$ is conflict-free, since only the explicit consideration of attacks gives rise to conflict in this approach. While this may sound peculiar according to the \virg{traditional} view, it is again coherent with the central role played by attacks and, as it will be seen later, does not prevent (indeed it enables) the achievement of the compatibility requirement with \AF.

\begin{examplecont}[\ref{ex_bob_afra}]
Consider $\unsetAFRA_1 = \set{\argg, \argc}$: as explained above, $\unsetAFRA_1$ is conflict-free as it does not explicitly include any attack. On the other hand, the sets $\unsetAFRA_2 = \set{\argg, \argc, \attackalpha}$, $\unsetAFRA_3 = \set{\argg, \argc, \attackbeta}$, $\unsetAFRA_4 = \set{\argg, \argc, \attackalpha, \attackbeta}$ are not conflict-free.
Note also that Definition \ref{conflict-free_gaf} encompasses sets consisting of attacks only. For instance the set $\unsetAFRA_5 = \set{\attackalpha, \attackbeta}$ is not conflict-free since $\AFRattacks{\attackalpha}{\attackbeta}$ ($\attackalpha$ indirectly defeats $\attackbeta$) and, analogously, $\AFRattacks{\attackbeta}{\attackalpha}$.
\end{examplecont}

\subsection{Acceptability and characteristic function}

The definition of acceptability is formally very similar to the traditional one, apart from the fact of encompassing sets of both arguments and attacks.

\begin{definition}[Acceptability]
\label{acceptability_gaf}
Let $\tup{\A}{\R}$ be an \GAF, $\SUB \subseteq \A \cup \R$ and $\elementw \in \A \cup \R$: 
$\elementw$ is acceptable w.r.t. $\SUB$ (or, equivalently \emph{is defended} by $\SUB$) iff $\forall \attackalpha \in \R$ s.t. $\AFRattacks{\attackalpha}{\elementw}$ $\exists \attackbeta \in \SUB$ s.t.
$\AFRattacks{\attackbeta}{\attackalpha}$.
\end{definition}

Note that while acceptability is defined with reference to a set $\SUB$ possibly including both arguments and attacks, only attacks are \virg{effective} as far as acceptability is concerned. In fact it is easy to see that an element (either argument or attack) is acceptable w.r.t. a set $\SUB$ if and only if it is acceptable w.r.t. to $\SUB \cap \R$.

\begin{examplecont}[\ref{ex_bob_afra}]
Considering \bobsAFRA, it can be seen that $\argg$ is acceptable w.r.t. $\set{\attackgamma}$ and w.r.t. $\set{\attackalpha}$, while it is not acceptable w.r.t. $\set{\argp}$. As other examples, $\attackbeta$ is acceptable w.r.t. $\set{\attackdelta}$, and $\attackgamma$ is acceptable w.r.t. $\set{\attackepsilon}$.
\end{examplecont}

Lemma \ref{lemma_acceptability_source_attack} shows that the acceptability of an attack implies the acceptability of its source, in accordance with the requirements mentioned in Section \ref{sec_motivations}.

\begin{lemma}
\label{lemma_acceptability_source_attack}
Let $\tup{\A}{\R}$ be an \AFRA{} and $\SUB \subseteq \A \cup \R$. If an attack $\attackalpha \in \R$ is acceptable w.r.t \SUB, then $\src{\attackalpha}$ is acceptable w.r.t to \SUB.

\begin{proof}
Suppose $\src{\attackalpha}=\arga{}$ is not acceptable w.r.t. \SUB. Then, $\exists \attackbeta$ s.t. $\AFRattacks{\attackbeta}{\arga}$ and $\nexists \attackgamma \in \SUB$ s.t. $\AFRattacks{\attackgamma}{\attackbeta}$. But since $\AFRattacks{\attackbeta}{\arga}$ and $\arga = \src{\attackalpha}$, then $\AFRattacks{\attackbeta}{\attackalpha}$; therefore $\attackalpha$ is not acceptable w.r.t. \SUB. Contradiction.
\end{proof}
\end{lemma}

The definition of characteristic function parallels the traditional one.

\begin{definition}
\label{def_characteristic_function_afra}
The characteristic function \charfunction{\anAFRA} of an \AFRA{} $\anAFRA = \tup{\A}{\R}$ is defined as follows:
$$
\charfunction{\anAFRA}: 2^{\A \cup \R} \mapsto 2^{\A \cup \R}
$$
$$
\charfunction{\anAFRA}(\unsetAFRA) = \{\elementv | \elementv \textrm{is acceptable w.r.t. }\unsetAFRA\}
$$
\end{definition}

Propositions \ref{prop_cffunction} and \ref{prop_monotonic} show that the fundamental properties of preserving conflict-freeness and being monotonic hold for the \AFRA{} characteristic function, as required.

\begin{proposition}
\label{prop_cffunction}
Let $\anAFRA=\tup{\A}{\R}$ be an \AFRA. If $\unsetAFRA \subseteq \A \cup \R$ is conflict-free, then $\charfunction{\anAFRA}(\unsetAFRA)$ is also conflict-free.
\end{proposition}
\begin{proof}
Assume that there are \attackalpha{} and \elementv{} in $\charfunction{\anAFRA}(\unsetAFRA)$ such that \AFRattacks{\attackalpha}{\elementv}. By the acceptability of $\elementv$, there exists $\attackbeta \in \unsetAFRA$ s.t. \AFRattacks{\attackbeta}{\attackalpha}. Then, by the acceptability of $\attackalpha$ there is $\attackbeta' \in \unsetAFRA$ s.t. \AFRattacks{\attackbeta'}{\attackbeta}, contradicting the hypothesis that $\unsetAFRA$ is conflict-free. Therefore $\charfunction{\anAFRA}(\unsetAFRA)$ is conflict-free.
\end{proof}

\begin{proposition}\label{prop_monotonic}
Let $\anAFRA=\tup{\A}{\R}$ be an \AFRA. The Function \charfunction{\anAFRA} is monotonic w.r.t. set inclusion.
\end{proposition}
\begin{proof}
Letting $\unsetAFRA \subseteq \unsetAFRA' \subseteq (\A \cup \R)$, we have to show that $\charfunction{\anAFRA}(\unsetAFRA) \subseteq \charfunction{\anAFRA}(\unsetAFRA')$, i.e. that every $\elementv$ which is acceptable w.r.t. \unsetAFRA{} is acceptable w.r.t. $\unsetAFRA'$. Suppose that $\elementv$ is acceptable w.r.t. \unsetAFRA{} but not w.r.t. $\unsetAFRA'$. Then, $\exists \attackalpha \in \R$ s.t. $\AFRattacks{\attackalpha}{\elementv}$ and $\nexists \attackbeta \in \unsetAFRA'$ s.t. $\AFRattacks{\attackbeta}{\attackalpha}$, which, since $\unsetAFRA \subseteq \unsetAFRA'$,  implies $\nexists \attackbeta \in \unsetAFRA$ s.t. $\AFRattacks{\attackbeta}{\attackalpha}$, which contradicts the hypothesis that \elementv{} is acceptable w.r.t. \unsetAFRA.
\end{proof}

\subsection{Admissibility}

The definition of admissible sets in \AFRA{} requires conflict-freeness and acceptability of all set elements, exactly as in \AF.

\begin{definition}[Admissibility]
\label{admissibility_gaf}
Let $\anAFRA = \tup{\A}{\R}$ be an \GAF: $\SUB \subseteq \A \cup \R$ is admissible iff it is conflict--free and each element of $\SUB$ is acceptable w.r.t. $\SUB$ (i.e. $\unsetAFRA \subseteq \charfunction{\anAFRA}(\unsetAFRA)$). 
\end{definition}

As required, a parallel of Dung's fundamental lemma holds in the context of \AFRA.

\begin{lemma}[Fundamental lemma]
\label{lemma_fundamental_lemma}
Let $\tup{\A}{\R}$ be an $\GAF$, $\SUB \subseteq \A \cup \R$ an admissible set and $\elementv, \elementv' \in \A \cup \R$ elements acceptable w.r.t. $\SUB$. Then:
\begin{enumerate}
\item $\SUB' = \SUB \cup \{\elementv\}$ is admissible; and
\item $\elementv'$ is acceptable w.r.t. $\SUB'$.
\end{enumerate}
\begin{proof}\mbox{}
\begin{enumerate}
\item $\elementv$ is acceptable w.r.t. $\SUB$ therefore each element of $\SUB'$ is acceptable w.r.t. $\SUB'$. Suppose $\SUB'$ is not conflict--free; therefore there exists an element $\elementw \in \SUB$ such that either $\AFRattacks{\elementv}{\elementw}$ or $\AFRattacks{\elementw}{\elementv}$. From the admissibility of $\SUB$ and the acceptability of $\elementv$ there exists an element $\bar{\elementw} \in \SUB$ such that $\AFRattacks{\bar{\elementw}}{\elementw}$ or $\AFRattacks{\bar{\elementw}}{\elementv}$. Since $\SUB$ is conflict--free it follows that $\AFRattacks{\bar{\elementw}}{\elementv}$. 
But then from the acceptability of $\elementv$ 
there must exist an element $\hat{\elementw} \in \SUB$ such that $\AFRattacks{\hat{\elementw}}{\bar{\elementw}}$. Contradiction.
\item Immediate from Proposition \ref{prop_monotonic}.
\qedhere
\end{enumerate}
\end{proof}
\end{lemma}

The following theorem completes the verification that \AFRA{} satisfies all the fundamental properties of Dung's theory listed in Proposition \ref{prop_requirements}.

\begin{theorem}
\label{thm_afra_dopo_fundamental_lemma}
 Let \anAFRA{} be an \AFRA. The set of all admissible sets of \anAFRA{} 
 forms a complete partial order with respect to set inclusion.
\begin{proof}
 We have to prove that (i) the set of all admissible sets has a least element and (ii)
 each chain of admissible sets has a least upper bound.
 Point (i) immediately follows from the fact that the empty set is admissible, 
 therefore it is obvioulsy the least element.
 As for (ii), let $\Omega$ be a chain of admissible sets: 
 we prove that $\SUB = \bigcup_{\omega \in \Omega}\omega$
 is admissible, thus obviously a least upper bound of $\Omega$.
 First, $\SUB$ is conflict-free, otherwise $\exists \elementu, \elementv \in \SUB$ 
 such that $\AFRattacks{\elementu}{\elementv}$, entailing that $\exists \omega \in \Omega$
 such that $\elementu, \elementv \in \omega$ and contradicting the admissibility of $\omega$.
 Second, suppose that $\elementu \in \SUB$ and $\AFRattacks{\elementv}{\elementu}$:
 we have to prove that $\exists \elementw \in \SUB$ such that $\AFRattacks{\elementw}{\elementv}$.
 The conclusion follows from the fact that $\exists \omega \in \Omega$ such that $\elementu \in \omega$,
 and since $\omega$ is admissible $\exists \elementw \in \omega \subseteq \SUB$ 
 such that $\AFRattacks{\elementw}{\elementv}$.
\end{proof}
\end{theorem}

\begin{examplecont}[\ref{ex_bob_afra}]
In \bobsAFRA{} there are fourty admissible sets, denoted in the following as $\ASAFRA{i}$.
First observe that according to Definition \ref{admissibility_gaf} the empty set is admissible for any \AFRA, thus we have $\ASAFRA{1} = \emptyset$. 
As to sets consisting of arguments only, note that only unattacked arguments can be admissible by themselves since in \AFRA{} defense is carried out by attack elements (for instance $\argg$ requires $\attackalpha$ for its defense). Thus we have $\ASAFRA{2} = \set{\argp}$, $\ASAFRA{3} = \set{\arga}$, and of course their union $\ASAFRA{6} = \set{\arga,$ $\argp}$ (the adopted numbering is in accordance with Figure \ref{fig_hasse}). 
Also singletons consisting of (directly or indirectly) unattacked attacks 
and those able to defend themselves on their own
are of course admissible, yielding $\ASAFRA{4} = \set{\attackalpha}$ and $\ASAFRA{5} = \set{\attackepsilon}$ 
(note for instance that $\attackdelta$ is indirectly defeated by $\attackepsilon$ and it does not defend itself).
Of course any set including only these individually admissible elements is admissible too, 
giving rise to 11 further admissible sets: 
$\ASAFRA{6} = \set{\arga,$ $\argp}$, 
$\ASAFRA{9} = \set{\argp,$ $\attackepsilon}$,
$\ASAFRA{10} = \set{\arga,$ $\attackepsilon}$,
$\ASAFRA{11} = \set{\argp,$ $\attackalpha}$,
$\ASAFRA{12} = \set{\arga,$ $\attackalpha}$,
$\ASAFRA{13} = \set{\attackepsilon,$ $\attackalpha}$,
$\ASAFRA{14} = \set{\argp,$ $\attackepsilon,$ $\attackalpha}$,
$\ASAFRA{15} = \set{\arga,$ $\attackepsilon,$ $\attackalpha}$,
$\ASAFRA{20} = \set{\arga,$ $\argp,$ $\attackalpha}$,
$\ASAFRA{21} = \set{\arga,$ $\argp,$ $\attackepsilon}$,
$\ASAFRA{34} = \set{\arga,$ $\argp,$ $\attackepsilon,$ $\attackalpha}$.

Considering now defense by individually admissible attacks we note that $\attackepsilon$ defends $\attackgamma$ by indirectly defeating $\attackdelta$ and $\attackalpha$ defends $\argg$ by indirectly defeating $\attackbeta$, leading to 
$\ASAFRA{7} = \set{\attackepsilon,$ $\attackgamma}$,
$\ASAFRA{8} = \set{\argg,$ $\attackalpha}$.
Of course the union of these two sets, being conflict-free, is admissible too,
leading to $\ASAFRA{32} = \set{\argg,$ $\attackepsilon,$ $\attackgamma,$ $\attackalpha}$.
Adding other unattacked elements to any of these three sets preserves admissibility, leading to the following 14 admissible sets:
$\ASAFRA{16} = \set{\argg,$ $\attackepsilon,$ $\attackalpha}$,
$\ASAFRA{17} = \set{\argp,$ $\argg,$ $\attackalpha}$,
$\ASAFRA{18} = \set{\arga,$ $\argg,$ $\attackalpha}$,
$\ASAFRA{22} = \set{\argp,$ $\attackepsilon,$ $\attackgamma}$,
$\ASAFRA{23} = \set{\arga,$ $\attackepsilon,$ $\attackgamma}$,
$\ASAFRA{24} = \set{\attackepsilon,$ $\attackgamma,$ $\attackalpha}$,
$\ASAFRA{25} = \set{\argp,$ $\argg,$ $\attackepsilon,$ $\attackalpha}$,
$\ASAFRA{26} = \set{\arga,$ $\argg,$ $\attackepsilon,$ $\attackalpha}$,
$\ASAFRA{33} = \set{\arga,$ $\argp,$ $\argg,$ $\attackalpha}$,
$\ASAFRA{27} = \set{\arga,$ $\argp,$ $\attackepsilon,$ $\attackgamma}$,
$\ASAFRA{31} = \set{\arga,$ $\attackepsilon,$ $\attackgamma,$ $\attackalpha}$,
$\ASAFRA{30} = \set{\argp,$ $\attackepsilon,$ $\attackgamma,$ $\attackalpha}$.
$\ASAFRA{37} = \set{\arga,$ $\argp,$ $\attackepsilon,$ $\attackgamma,$ $\attackalpha}$,
$\ASAFRA{39} = \set{\arga,$ $\argp,$ $\argg,$ $\attackepsilon,$ $\attackalpha}$.

Since $\attackgamma$, being defended by $\attackepsilon$, 
in turn defends $\argg$ by directly defeating $\attackbeta$, the set
$\ASAFRA{19} = \set{\argg,$ $\attackepsilon,$ $\attackgamma}$ is admissible.
Again, adding unattacked elements gives rise to the following 6 further admissible sets:
$\ASAFRA{28} = \set{\argp,$ $\argg,$ $\attackepsilon,$ $\attackgamma}$,
$\ASAFRA{29} = \set{\arga,$ $\argg,$ $\attackepsilon,$ $\attackgamma}$,
$\ASAFRA{35} = \set{\argp,$ $\argg,$ $\attackepsilon,$ $\attackgamma,$ $\attackalpha}$,
$\ASAFRA{36} = \set{\arga,$ $\argg,$ $\attackepsilon,$ $\attackgamma,$ $\attackalpha}$,
$\ASAFRA{38} = \set{\arga,$ $\argp,$ $\argg,$ $\attackepsilon,$ $\attackgamma}$,
$\ASAFRA{40} = \set{\arga,$ $\argp,$ $\argg,$ $\attackepsilon,$ $\attackgamma,$ $\attackalpha}$.

Figure \ref{fig_hasse} shows the Hasse diagram (w.r.t. set inclusion) of the admissible sets listed above. Coherently with Theorem \ref{thm_afra_dopo_fundamental_lemma} this is a complete partial order with the empty set as minimal element at the bottom (as for any \AFRA{} and for any \AF) and (at least) one maximal admissible set, namely $\ASAFRA{40}$.
\end{examplecont}

\begin{figure}[ht]
	\centering
	\includegraphics[width=11cm]{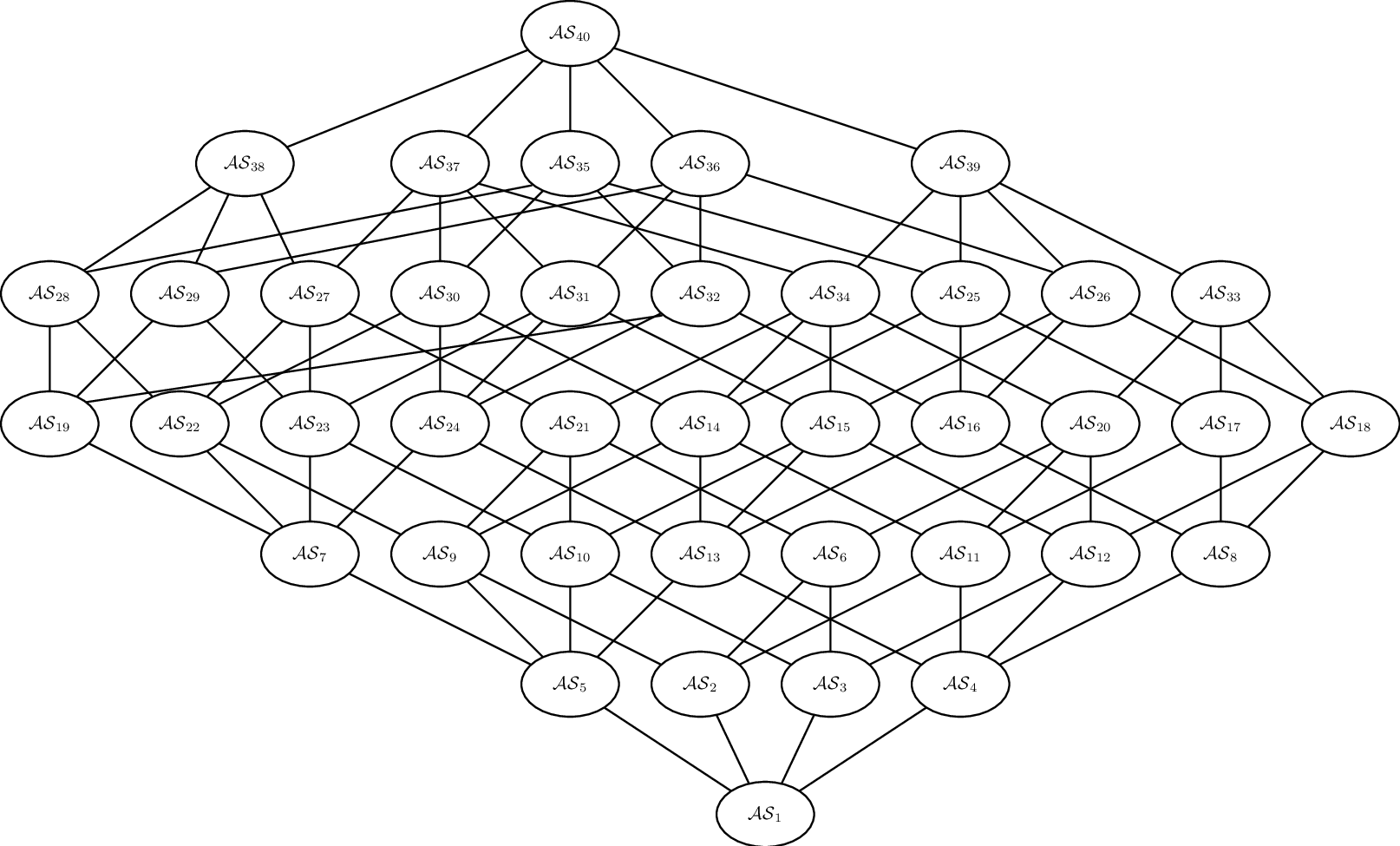}
	\caption{Hasse diagram of admissible sets for Example \ref{ex_bob_afra}.}
	\label{fig_hasse}
\end{figure}

\section{Semantics for \AFRA}
\label{sec_semantics_afra_2}

In this section we define and analyse the \AFRA{} semantics corresponding to the ones listed in Definition \ref{def_sem_recall}.

\subsection{Complete Semantics}

The notion of complete extension closely parallels the traditional one by requiring admissibility and the inclusion of any acceptable argument.

\begin{definition}[Complete extension]
\label{def_complete_afra}
Let $\anAFRA = \tup{\A}{\R}$ be an \AFRA. A set $\unsetAFRA \subseteq \A \cup \R$ is a complete extension if and only if \unsetAFRA{} is admissible and every element of $\A \cup \R$ which is acceptable w.r.t. \unsetAFRA{} belongs to \unsetAFRA, i.e. $\charfunction{\anAFRA}(\unsetAFRA) \subseteq \unsetAFRA$. 
\end{definition}

By inspection of Definitions \ref{admissibility_gaf} and \ref{def_complete_afra} it is immediate to see that a complete extension can be equivalently characterized as a conflict-free set $\unsetAFRA$ which is a fixed point of $\charfunction{\anAFRA}$, i.e. such that $\charfunction{\anAFRA}(\unsetAFRA) =\unsetAFRA$.

\begin{examplecont}[\ref{ex_bob_afra}]
In \bobsAFRA{} there is exactly one complete extension: $\set{\arga,$ $\argp,$ $\argg,$ $\attackepsilon,$ $\attackgamma,$ $\attackalpha}$.
\end{examplecont}

We introduce also a more articulated \AFRA{} (shown in Figure \ref{fig_ex_semantics}) which will be useful for illustration and comparison of the semantics to be introduced in the following.

\begin{example}
\label{ex_semantics}
Let $\anotherAFRA = \tup{\A}{\R}$ be an $\AFRA$, where:
$\A = \set{\arga,$ $\argb,$ $\argc,$ $\argd,$ $\arge,$ $\argf,$ $\argg}$, and
$\R = \set{\attackalpha,$ $\attackbeta,$ $\attackgamma,$ $\attackdelta,$ $\attackepsilon,$ $\attacketa,$ $\attackzeta,$ $\attacktheta,$ $\attackiota,$ $\attackkappa}$ with
$\attackalpha = (\arga, \argb)$,
$\attackbeta = (\argb, \attackalpha)$,
$\attackgamma = (\argc, \attackalpha)$,
$\attackdelta = (\argc, \argd)$,
$\attackepsilon = (\arge, \attackdelta)$,
$\attacketa = (\argd, \attackepsilon)$,
$\attackzeta = (\arga, \argf)$,
$\attacktheta = (\argf, \arga)$,
$\attackiota = (\argf, \argg)$,
$\attackkappa = (\argg, \argg)$.

As to the complete extensions of $\anotherAFRA$, note first that the unattacked elements are $\arge$, $\argc$ and $\attackgamma$ and that $\attackgamma$ defends both $\argb$ and $\attackbeta$ by directly defeating $\attackalpha$. It follows that $\set{\argb,$ $\argc,$ $\arge,$ $\attackbeta,$ $\attackgamma}$ is a complete extension.
Further note that $\attacktheta$ defends itself, $\argf$ and $\attackiota$ by indirectly defeating $\attackzeta$ and, analogously, $\attackzeta$ defends itself and $\arga$ by indirectly defeating $\attacktheta$. This gives rise to two further complete extensions, namely 
$\set{\argb,$ $\argc,$ $\arge,$ $\argf,$ $\attackbeta,$ $\attackgamma,$ $\attacktheta,$ $\attackiota}$ and
$\set{\arga,$ $\argb,$ $\argc,$ $\arge,$ $\attackbeta,$ $\attackgamma,$ $\attackzeta}$.
All other arguments and attacks in $\anotherAFRA$ have no defense and hence do not belong to any admissible set or complete extension.
\end{example}

\begin{figure}[ht]
	\centering
	\includegraphics[scale=0.25]{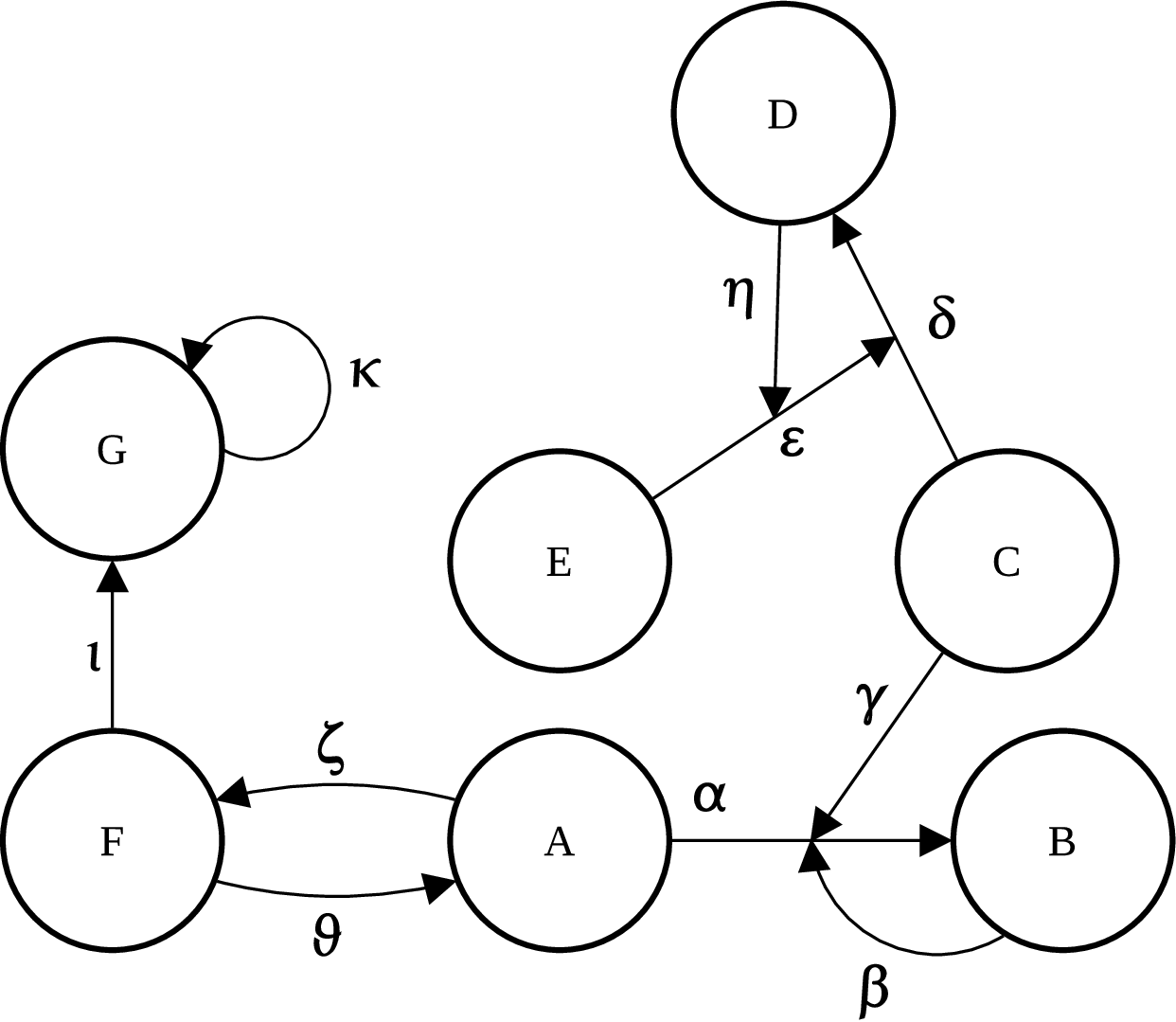}
	\caption{Graphical representation of Example \ref{ex_semantics}.}
	\label{fig_ex_semantics}
\end{figure}

\subsection{Grounded Semantics}

The definition of grounded semantics parallels Dung's one: as in his approach, the basic properties of the characteristic function (whose validity we have already proved also in the context of \AFRA) ensure the uniqueness of the grounded extension and the fact that it can be equivalently characterized as the least complete extension.

\begin{definition}[Grounded extension]
Let $\anAFRA = \tup{\A}{\R}$ be an $\AFRA$. 
The grounded extension of \anAFRA{} 
is the least fixed point of \charfunction{\anAFRA}.
\end{definition}

\begin{lemma}
\label{lemma_grounded_as_complete}
The grounded extension is the least complete extension.
\end{lemma}

The identification of the grounded extension in Examples \ref{ex_bob_afra} and \ref{ex_semantics} follows easily.

\begin{examplecont}[\ref{ex_bob_afra}]
The grounded extension of \bobsAFRA{} is $\set{\arga,$ $\argp,$ $\argg,$ $\attackepsilon,$ $\attackgamma,$ $\attackalpha}$.
\end{examplecont}

\begin{examplecont}[\ref{ex_semantics}]
The grounded extension of $\anotherAFRA$ is $\set{\argb,$ $\argc,$ $\arge,$ $\attackbeta,$ $\attackgamma}$.
\end{examplecont}

\subsection{Preferred semantics}

As expected, preferred extensions are defined as maximal admissible sets.

\begin{definition}[Preferred extension]
\label{def_preferred_afra}
Let $\anAFRA = \tup{\A}{\R}$ be an \GAF. 
A set $\SUB \subseteq \A \cup \R$ is a preferred extension of \anAFRA{} iff it is a maximal (w.r.t. set inclusion) admissible set. 
\end{definition}

Theorem \ref{thm_adm_in_pref} and Corollary \ref{cor_pref_exist} follow directly from Theorem \ref{thm_afra_dopo_fundamental_lemma}.

\begin{theorem}\label{thm_adm_in_pref}
Let $\anAFRA = \tup{\A}{\R}$ be an \GAF. For each admissible set \SUB{} of \anAFRA, there exists a preferred extension \E{} of \anAFRA{} such that $\SUB \subseteq \E$.
\end{theorem}

\begin{corollary}\label{cor_pref_exist}
Every \AFRA{} possesses at least one preferred extension.
\end{corollary}

It also holds that preferred extensions are complete (and hence can be equivalently characterized as maximal complete extensions).

\begin{lemma}
\label{lemma_preferred_as_complete}
Every preferred extension is a complete extension, but not vice versa.
\begin{proof}
Let $\SUB$ be a preferred extension which is not complete, 
then $\exists \elementv \notin \SUB$ which is acceptable w.r.t. $\SUB$ and by the fundamental lemma 
(Lemma \ref{lemma_fundamental_lemma}) $\SUB \cup \set{\elementv}$ is admissible: 
but this contradicts the maximality of $\SUB$. 
As to the other point, in Example \ref{ex_semantics} one of the complete extensions is not preferred (see below).
\end{proof}
\end{lemma}

Maximal complete extensions are easily identified in Examples \ref{ex_bob_afra} and \ref{ex_semantics}.

\begin{examplecont}[\ref{ex_bob_afra}]
In \bobsAFRA{} the only preferred extension is the grounded extension, \ie{} 
$\set{\arga,$ $\argp,$ $\argg,$ $\attackepsilon,$ $\attackgamma,$ $\attackalpha}$.
\end{examplecont}

\begin{examplecont}[\ref{ex_semantics}]
The preferred extensions of $\anotherAFRA$ are 
$\set{\argb,$ $\argc,$ $\arge,$ $\argf,$ $\attackbeta,$ $\attackgamma,$ $\attacktheta,$ $\attackiota}$
and
$\set{\arga,$ $\argb,$ $\argc,$ $\arge,$ $\attackbeta,$ $\attackgamma,$ $\attackzeta}$, while the complete (and grounded) extension $\set{\argb,$ $\argc,$ $\arge,$ $\attackbeta,$ $\attackgamma}$ is not preferred since it is not maximal w.r.t. set inclusion.
\end{examplecont}

\subsection{Stable semantics}

Stable semantics is based, as usual, on the idea that each extension attacks all elements not included in it.

\begin{definition}[Stable extension]\label{def_stab_afra}
Let $\anAFRA = \tup{\A}{\R}$ be an $\AFRA$. A set $\SUB \subseteq \A \cup \R$ is a stable extension of \anAFRA{} if and only if \SUB{} is conflict-free and $\forall \elementv \in \A \cup \R, \elementv \notin \SUB$, $\exists \attackalpha \in \SUB$ s.t. $\AFRattacks{\attackalpha}{\elementv}$. 
\end{definition}

Stable extensions are also preferred, but not vice versa. In particular, as in \AF, there are cases where no extensions complying with Definition \ref{def_stab_afra} exist.

\begin{lemma}
\label{lemma_afra_stable_preferred}
Every stable extension is a preferred extension, but not vice versa.
\end{lemma}
\begin{proof}
It is easy to see that each stable extension is a maximal complete extension, hence a preferred extension. To show that the reverse does not hold, consider an \AFRA{} consisting just of a self-defeating argument: $\anAFRA{} = \tup{\A}{\R}$ with $\A = \{\arga\}$, $\R = \{(\arga, \arga)\}$. The empty set is a preferred extension of \anAFRA{} but clearly is not stable.
\end{proof}

\begin{examplecont}[\ref{ex_bob_afra}]
The only stable extension of $\bobsAFRA$ is $\set{\arga,$ $\argp,$ $\argg,$ $\attackepsilon,$ $\attackgamma,$ $\attackalpha}$.
\end{examplecont}

\begin{examplecont}[\ref{ex_semantics}]
The two preferred extensions of $\anotherAFRA$ are not stable. In particular neither of them includes nor defeats the elements $\attackdelta$, $\attackepsilon$, $\attacketa$, and $\argd$.
\end{examplecont}

\subsection{Semi-stable semantics}

Semi-stable semantics \cite{caminada2006} is based on the idea of prescribing the maximization of both the arguments included in an extension and those attacked by it, i.e. of maximizing the extension range.

\begin{definition}[Range]
\label{def_range_afra}
Let $\anAFRA = \tup{\A}{\R}$ be an \AFRA{} 
and let $\unsetAFRA \subseteq \A \cup \R$ be a set of arguments and attacks. 
The range of $\unsetAFRA$, denoted as $\range{\unsetAFRA}$, 
is defined as $\unsetAFRA \cup \unsetAFRA^{+}$ where $\unsetAFRA^{+} = \set{\elementv \in \A \cup \R | \exists \attackalpha \in \unsetAFRA \textrm{ s.t. } \AFRattacks{\attackalpha}{\elementv}}$.
\end{definition}

\begin{definition}[Semi-stable extension]
\label{def_semi_afra}
Let $\anAFRA = \tup{\A}{\R}$ be an \AFRA, a set $\unsetAFRA \subseteq \A \cup \R$ is a semi-stable extension iff $\unsetAFRA$ is a complete extension with maximal (w.r.t set inclusion) range. 
\end{definition}

Proposition \ref{prop_afra_semi} summarizes the relations of semi-stable with stable and preferred semantics in \AFRA, paralleling those holding in \AF.

\begin{proposition}
\label{prop_afra_semi}
For any $\AFRA$ $\anAFRA = \tup{\A}{\R}$
\begin{enumerate}
\item if stable extensions exist, then they coincide with semi-stable extensions;
\item every semi-stable extension is preferred but not viceversa.
\end{enumerate}
\end{proposition}
\begin{proof}
As to the first point, note that, by definition, the range of any stable extension coincides with $\A \cup \R$, which is of course the largest possible one. Moreover stable extensions are admissible sets by Lemma \ref{lemma_afra_stable_preferred}, hence the conclusion.
As to the second point, suppose a semi-stable extension $\unsetAFRA$ is not preferred, i.e. there is an admissible set $\unsetAFRA'$ strictly including it: the range of $\unsetAFRA'$ strictly includes the range of $\unsetAFRA$, contradicting the hypothesis that $\unsetAFRA$ is a semi-stable extension.
On the other hand there are preferred extensions which are not semi-stable as in Example \ref{ex_semantics} (see below).
\end{proof}

\begin{examplecont}[\ref{ex_bob_afra}]
The only semi-stable extension of \bobsAFRA{} is $\set{\arga,$ $\argp,$ $\argg,$ $\attackepsilon,$ $\attackgamma,$ $\attackalpha}$.
\end{examplecont}

\begin{examplecont}[\ref{ex_semantics}]
Consider the preferred extensions of $\anotherAFRA$. Letting $\unsetAFRA=\set{\argb,$ $\argc,$ $\arge,$ $\argf,$ $\attackbeta,$ $\attackgamma,$ $\attacktheta,$ $\attackiota}$ it holds $\unsetAFRA^{+}=\set{\argg, \arga, \attackalpha, \attackzeta, \attackkappa}$. On the other hand letting $\unsetAFRA_{*}=\set{\arga,$ $\argb,$ $\argc,$ $\arge,$ $\attackbeta,$ $\attackgamma,$ $\attackzeta}$ it holds $\unsetAFRA_{*}^{+}=\set{\argf,\attackalpha, \attacktheta, \attackiota}$. $\unsetAFRA$ is the only semi-stable extension of $\anotherAFRA$ since $(\unsetAFRA \cup \unsetAFRA^{+}) \supsetneq (\unsetAFRA_{*} \cup \unsetAFRA_{*}^{+})$.
\end{examplecont}

\subsection{Ideal semantics}

Ideal semantics \cite{dungetal2006} considers the largest admissible set included in all preferred extensions.

\begin{definition}[Ideal extension]
\label{def_ideal_afra}
Let $\anAFRA = \tup{\A}{\R}$ be an \AFRA. A set $\unsetAFRA \subseteq \A \cup \R$ is ideal iff $\unsetAFRA$ is admissible and $\forall \aPreferredExtension$ s.t. $\aPreferredExtension$ is a preferred extension of $\anAFRA$, $\unsetAFRA \subseteq \aPreferredExtension$. The ideal extension is the maximal (w.r.t. set inclusion) ideal set.
\end{definition}

Definition \ref{def_ideal_afra} anticipates the uniqueness of ideal extension shown in the following proposition.

\begin{proposition}
The ideal extension is unique.
\end{proposition}
\begin{proof}
Suppose that there are two distinct maximal ideal sets $\unsetAFRA$ and $\unsetAFRA'$ complying with Definition \ref{def_ideal_afra}.
Now $\unsetAFRA \cup \unsetAFRA'$ is included in all the preferred extensions, hence it is conflict-free, and defends all its elements, hence it is also admissible. Therefore $\unsetAFRA \cup \unsetAFRA'$ is a larger ideal set than $\unsetAFRA$ and $\unsetAFRA'$, contradicting the hypothesis.
\end{proof}

It can also be seen that the ideal extension includes all acceptable elements,
\ie{} it is a complete extension.

\begin{lemma}
\label{lemma_ideal_as_complete}
 The ideal extension is a complete extension.
\begin{proof}
 The ideal extension is admissible by definition, 
 thus it is sufficient to show that it includes any element $\elementv$
 which is acceptable w.r.t. it.
 Since the ideal extension is contained in any preferred extension, 
 it is easy to see that $\elementv$ is acceptable w.r.t any preferred extension too.
 By Lemma \ref{lemma_preferred_as_complete} the preferred extensions are also complete,
 therefore they must all include $\elementv$.
 As a consequence, $\elementv$ is included in the ideal extension, otherwise including it
 would give rise by the fundamental lemma (Lemma \ref{lemma_fundamental_lemma}) 
 to a strictly greater admissible set
 contained in all preferred extensions, contradicting the maximality of the ideal extension.
\end{proof}
\end{lemma}

Since the grounded extension is included in all complete extensions (and hence in all preferred extensions) and is admissible, the ideal extension is a (possibly strict) superset of the grounded extension.

\begin{examplecont}[\ref{ex_bob_afra}]
The ideal extension of \bobsAFRA{} is $\set{\arga,$ $\argp,$ $\argg,$ $\attackepsilon,$ $\attackgamma,$ $\attackalpha}$.
\end{examplecont}

\begin{examplecont}[\ref{ex_semantics}]
The ideal extension of $\anotherAFRA$ is $\set{\argb,$ $\argc,$ $\arge,$ $\attackbeta,$ $\attackgamma}$.
\end{examplecont}

\section{Compatibility with \AF}
\label{sec_compatibility_afra_af}

In this section we prove the satisfaction of the compatibility requirement formally stated at the end of Section \ref{sec_motivations}  for the case where a given \GAF{} is an \AF. To be precise, throughout this section when stating \virg{let $\anAFRA = \tup{\A}{\R}$ be an \AF{}}, we will consider that $\anAFRA$ is an \AFRA{} such that attacks involve just arguments rather than being directed against other attacks (formally, $\R \subseteq \A \times \A$).

First of all, it is easy to see that, in this case, a dual property holds 
w.r.t. Lemma \ref{lemma_acceptability_source_attack}.

\begin{lemma}
  \label{lemma_af_acceptability_attack_source}
  Let $\anAFRA = \tup{\A}{\R}$ be an \AF{} and $\SUB \subseteq (\A \cup \R)$. 
  If an argument $\arga \in \A$ is acceptable w.r.t. \SUB, 
  then any $\attackalpha \in \R$ such that $\src{\attackalpha} = \arga$
  is acceptable w.r.t. \SUB.
\begin{proof}
  We prove that \attackalpha{} is defended by \SUB{} from any attack.
  For any \attackbeta{} such that $\AFRattacks{\attackbeta}{\attackalpha}$, 
  \attackbeta{} does not directly attack \attackalpha{} since $\tup{\A}{\R}$ is an \AF.
  As a consequence, it must be the case that $\AFRattacks{\attackbeta}{\src{\attackalpha}}$,
  \ie{} $\AFRattacks{\attackbeta}{\arga}$: since \arga{} is acceptable w.r.t. \SUB, 
  then $\exists \attackgamma \in \SUB$ s.t. $\AFRattacks{\attackgamma}{\attackbeta}$.
\end{proof}
\end{lemma}

As anticipated above, we show that in the case of an \AF{}
the semantics defined in Section \ref{sec_semantics_afra_2}
reduce to those adopted in the context of the traditional Dung's framework.
Of course, this correspondence can only be established through a mapping,
since extensions in \AFRA, differently from those in the traditional Dung's framework, 
include both arguments and attacks. Accordingly, Definition \ref{defafraoperator} provides a natural way to extend sets of arguments (corresponding to traditional extensions) into sets of arguments and attacks (corresponding to \AFRA{} extensions).

\begin{definition}[\opexttoafra{} operator]\label{defafraoperator}
  Let $\anAFRA = \tup{\A}{\R}$ be an \AFRA. Given a set of arguments $\U \subseteq \A$,
  $\exttoafra{\U} \triangleq \U \cup \set{\attackalpha \in \R \mid \src{\attackalpha} \in \U}$.
\end{definition}

In words, given a set of arguments $\U \subseteq \A$, the \opexttoafra{} operator 
completes $\U$ with all of the attacks arising from it. This operator will play a key role in proving the satisfaction of compatibility requirements for all the considered semantics (Propositions \ref{prop_corr_af_complete}-\ref{prop_corr_af_ideal}). In fact, given a semantics $\gensem$, the compatibility requirement (in the case where a given \GAF{} is an \AF) will be expressed as a bijective correspondence, through the \opexttoafra{} operator, between (i) extensions prescribed by $\gensem$ in the traditional \AF{} formulation, and (ii) extensions prescribed by $\gensem$ in the \GAF{} formulation.
More specifically, we will show for each semantics $\gensem$ that if a set of arguments $\U$ is an \AF{} extension according to $\gensem$ then $\exttoafra{\U}$ is an \GAF{} extension according to $\gensem$ and, vice versa, if a set of arguments and attacks $\SUB$ is an \GAF{} extension according to $\gensem$ then there is a set of arguments $\U$ such that $\SUB=\exttoafra{\U}$ and $\U$ is an \AF{} extension according to $\gensem$.

The use of the \opexttoafra{} operator to prove these correspondences is supported by Lemmata \ref{lemma_acceptability_source_attack} and \ref{lemma_af_acceptability_attack_source}.
The relevant properties shown in Lemma \ref{lemma_afra_operator_properties} 
will be exploited in the following.

\begin{lemma}
\label{lemma_afra_operator_properties}
  Let $\anAFRA = \tup{\A}{\R}$ be an \GAF, and let $\U_1, \U_2 \subseteq \A$ two sets of arguments. It holds that:
  \begin{enumerate}
    \item $\U_1 \subseteq \U_2$ \ifff{} $\exttoafra{\U_1} \subseteq \exttoafra{\U_2}$
    \item $\U_1 \subsetneq \U_2$ \ifff{} $\exttoafra{\U_1} \subsetneq \exttoafra{\U_2}$
    \item $\exttoafra{\U_1} \cup \exttoafra{\U_2} = \exttoafra{(\U_1 \cup \U_2)}$
  \end{enumerate}
\begin{proof}
  \mbox{ }
  \begin{enumerate}
    \item As for the $\Rightarrow$ direction, let $\elementv \in \exttoafra{\U_1}$. 
          If \elementv{} is an argument then by definition $\elementv \in \U_1 \subseteq \U_2$
          thus $\elementv$ also belongs to $\exttoafra{\U_2}$. In the other case $\elementv \in \R$
          and by definition $\exists \arga \in \U_1 : \arga = \src{\elementv}$,
          thus $\arga \in \U_2$ and, again by definition, $\elementv \in \exttoafra{\U_2}$.
          As for the other direction, if an argument $\arga \in \U_1$ 
          then it belongs to $\exttoafra{\U_1} \subseteq \exttoafra{\U_2}$, 
          and since it is an argument then it must be the case that $\arga \in \U_2$. 
    \item Taking into account the previous point, for the $\Rightarrow$ direction we have just to show that,
          considering an argument $\arga \in \U_2$ such that $\arga \notin \U_1$,
          it holds by definition that $\arga \in \exttoafra{\U_2}$ but $\arga \notin \exttoafra{\U_2}$,
          entailing that $\exttoafra{\U_1} \subsetneq \exttoafra{\U_2}$. As for the other direction,
          by the hypothesis $\exists \elementv \in \exttoafra{\U_2} : \elementv \notin \exttoafra{\U_1}$.
          If $\elementv$ is an argument then by definition $\elementv \in \U_2$ and $\elementv \notin \U_1$,
          if $\elementv \in \R$ then these conditions hold for $\src{\elementv}$: 
          in any case, $\U_1 \subsetneq \U_2$.
    \item $\elementv \in (\exttoafra{\U_1} \cup \exttoafra{\U_2})$ $\Leftrightarrow$
    			$\elementv \in \exttoafra{\U_1}$ or $\elementv \in \exttoafra{\U_2}$ $\Leftrightarrow$
    			$\elementv \in \U_1 \vee \exists \arga \in \U_1 : \arga = \src{\elementv}
    			\vee \elementv \in \U_2 \vee \exists \argb \in \U_2 : \argb = \src{\elementv}$ $\Leftrightarrow$
    			$\elementv \in (\U_1 \cup \U_2) \vee \exists \argc \in (\U_1 \cup \U_2) : \argc = \src{\elementv}$ $\Leftrightarrow$
    			$\elementv \in \exttoafra{(\U_1 \cup \U_2)}$.
  \end{enumerate}
\end{proof}
\end{lemma}

A key role in proving the satisfaction of the compatibility requirement is played by showing in Proposition \ref{prop_corr_af_complete} that the desired bijective correspondence between \GAF{} and \AF{} extensions holds for the case of complete semantics.

\begin{proposition}
\label{prop_corr_af_complete}
  Let $\anAFRA = \tup{\A}{\R}$ be an \AF. 
  Then, $\SUB$ is a complete extension of $\anAFRA$ \ifff{} 
  $\SUB = \exttoafra{\U}$ where $\U$ is a \dungcomplete{} extension of $\anAFRA$.
\begin{proof}
  $\Rightarrow$. We first show that $\SUB = \exttoafra{(\SUB \cap \A)}$.
  In fact, for any $\elementv \in \SUB$ 
  if $\elementv \in \A$ then it obviously belongs to $\exttoafra{(\SUB \cap \A)}$.
  In the other case, namely $\elementv \in \R$, $\elementv$ is acceptable w.r.t. $\SUB$ 
  since it belongs to $\SUB$ which is a complete extension,
  thus by Lemma \ref{lemma_acceptability_source_attack} $\src{\elementv} \in (\SUB \cap \A)$,
  which by definition of the \opexttoafra{} operator entails $\elementv \in \exttoafra{(\SUB \cap \A)}$.
  On the other hand, for any $\elementv \in \exttoafra{(\SUB \cap \A)}$
  if $\elementv \in \A$ then it obviously belongs to $\SUB$;
  in the other case $\src{\elementv} \in \SUB$ and $\elementv \in \SUB$ follows from
  Lemma \ref{lemma_af_acceptability_attack_source} and the fact that $\SUB$ is a complete extension. \\
  According to this result, we have to show that $(\SUB \cap \A)$ is a \dungcomplete{} extension. \\
  First, $(\SUB \cap \A)$ is \dungconffree, 
  otherwise there would exist $\arga, \argb \in (\SUB \cap \A)$ with $\attacksAF{\arga}{\argb}$,
  \ie{} letting $\attackalpha = (\arga, \argb)$ we would have 
  $\AFRattacks{\attackalpha}{\argb}$ with $\src{\attackalpha} = \arga$:
  since both $\argb$ and $\attackalpha$ belong to $\SUB = \exttoafra{(\SUB \cap \A)}$, $\SUB$ 
  would not be conflict-free, contradicting the hypothesis. \\
  Then, we show that $(\SUB \cap \A)$ is \dungadmissible, \ie{}
  given $\arga \in (\SUB \cap \A)$, for any $\argb \in \A$ such that $\attacksAF{\argb}{\arga}$
  $\exists \argc \in (\SUB \cap \A)$ such that $\attacksAF{\argc}{\argb}$.
  Since $\attacksAF{\argb}{\arga}$, letting $\attackalpha = (\argb, \arga)$
  yields $\AFRattacks{\attackalpha}{\arga}$ with $\src{\attackalpha} = \argb$.
  Since $\arga \in \SUB$ and $\SUB$ is admissible by the hypothesis,
  $\exists \attackbeta \in \SUB : \AFRattacks{\attackbeta}{\attackalpha}$,
  which taking into account that $\anAFRA$ is an \AF{} yields $\AFRattacks{\attackbeta}{\argb}$.
  Since $\attackbeta \in \SUB = \exttoafra{(\SUB \cap \A)}$, $\src{\attackbeta} \in \SUB$,
  and the thesis follows from $\attacksAF{\src{\attackbeta}}{\argb}$. \\
  Finally, we prove that $(\SUB \cap \A)$ is \dungcomplete{} by showing that,
  for any $\arga \in \A$ which is \dungacceptable{} w.r.t. $(\SUB \cap \A)$,
  $\arga$ is acceptable w.r.t. $\SUB$: 
  by the hypothesis that $\SUB$ is a complete extension it then follows that
  $\arga \in \SUB$, \ie{} $\arga \in (\SUB \cap \A)$.
  Let us then consider an attack $\attackalpha \in \R$ such that $\AFRattacks{\attackalpha}{\arga}$.
  Obviously this is equivalent to $\attacksAF{\src{\attackalpha}}{\arga}$,
  and since $\arga$ is \dungacceptable{} w.r.t. $(\SUB \cap \A)$,
  $\exists \argb \in (\SUB \cap \A) : \attacksAF{\argb}{\src{\attackalpha}}$.
  Letting $\attackbeta = (\argb, \src{\attackalpha})$,
  we have $\AFRattacks{\attackbeta}{\attackalpha}$,
  and since $\argb \in (\SUB \cap \A)$ 
  then also $\attackbeta \in \SUB = \exttoafra{(\SUB \cap \A)}$.
  Summing up, for any $\attackalpha \in \R$ such that $\AFRattacks{\attackalpha}{\arga}$
  $\exists \attackbeta \in \SUB$ such that $\AFRattacks{\attackbeta}{\attackalpha}$.
  
  $\Leftarrow$. We have to show that $\exttoafra{\U}$ is a complete extension, namely conflict-free, 
  admissible and including all acceptable elements. \\
  As to the first point, assume by contradiction that $\exists \attackalpha, \elementv \in \exttoafra{\U}$
  such that $\AFRattacks{\attackalpha}{\elementv}$.
  While $\attackalpha \in \R$ by definition, $\elementv$ either belongs to $\A$ or to $\R$.
  In the first case we have $\attacksAF{\src{\attackalpha}}{\elementv}$,
  and by definition of $\exttoafra{\U}$ both $\src{\attackalpha}$ and $\elementv$ belong to $\U$,
  contradicting the fact that $\U$ is \dungconffree. 
  In the other case, \ie{} $\elementv$ is an attack, since  $\anAFRA$ is an \AF{}
  we have $\AFRattacks{\attackalpha}{\src{\elementv}}$, 
  \ie{} $\attacksAF{\src{\attackalpha}}{\src{\elementv}}$.
  But $\attackalpha, \elementv \in \exttoafra{\U}$ entails $\src{\attackalpha}, \src{\elementv} \in \U$,
  again contradicting the fact that $\U$ is \dungconffree. \\
  To show that $\exttoafra{\U}$ is admissible, consider a generic $\elementv \in \exttoafra{\U}$
  and suppose that $\exists \attackalpha \in \R$ such that $\AFRattacks{\attackalpha}{\elementv}$.
  Taking into account that $\anAFRA$ is an \AF{} and that by definition $\exttoafra{\U}$ includes the sources of
  all the attacks it includes, it is easy to see that $\exists \arga \in \U$ 
  such that $\AFRattacks{\attackalpha}{\arga}$ (where $\arga$ is either $\elementv$ or $\src{\elementv}$).
  Therefore $\attacksAF{\src{\attackalpha}}{\arga}$, and since $\U$ is \dungadmissible{}
  $\exists \argb \in \U : \attacksAF{\argb}{\src{\attackalpha}}$.
  Letting $\attackbeta = (\argb, \src{\attackalpha})$, 
  we have $\AFRattacks{\attackbeta}{\attackalpha}$ and by definition of $\exttoafra{\U}$
  it is the case that $\attackbeta \in \exttoafra{\U}$. \\
  Finally, we have to show that if $\elementv$ is acceptable w.r.t. $\exttoafra{\U}$
  then $\elementv \in \exttoafra{\U}$.
  If \elementv{} is an argument, taking into account that \anAFRA{} is an \AF, the conclusion follows from the \dungcomplete ness of \U. If \elementv{} is an attack, we prove that $\src{\elementv}$ is \dungacceptable{} w.r.t. $\U$,
  which, taking into account that $\U$ is a \dungcomplete{} extension, implies that $\src{\elementv} \in \U$,
  in turn entailing $\elementv \in \exttoafra{\U}$ by the definition of $\exttoafra{\U}$.
  Thus, assume that there is $\arga \in \A$ such that $\attacksAF{\arga}{\src{\elementv}}$.
  Letting $\attackalpha = (\arga, \src{\elementv})$, $\AFRattacks{\attackalpha}{\elementv}$
  and since $\elementv$ is acceptable w.r.t. $\exttoafra{\U}$ which is admissible by the previous point,
  there is $\attackbeta \in \exttoafra{\U}$ such that $\AFRattacks{\attackbeta}{\attackalpha}$.
  Since $\anAFRA$ is an \AF{} it must be the case that $\AFRattacks{\attackbeta}{\arga}$,
  obviously entailing that $\attacksAF{\src{\attackbeta}}{\arga}$
  with $\src{\attackbeta} \in \U$ by the definition of $\exttoafra{\U}$.
\end{proof}
\end{proposition}

To exemplify this correspondence, consider the following example graphically represented in Figure \ref{fig:afraasaf}.

\begin{example}
\label{ex_afraasaf}
Let $\athirdAFRA = \tup{\set{A,B,C,D}}{\set{\alpha,\beta,\gamma,\delta, \epsilon, \zeta}}$ be an \GAF{} where: $\alpha=(A,B)$, $\beta=(B,A)$, $\gamma=(A,C)$, $\delta=(B,C)$, $\epsilon=(C,D)$, $\zeta=(D,C)$. \athirdAFRA{} is clearly also an \AF.
The \dungcomplete{} extensions of $\athirdAFRA$ are $\emptyset$, $\set{D}$, $\set{A, D}$, $\set{B, D}$ while the complete extensions (as defined in \AFRA) are $\emptyset$, $\set{D, \zeta}$, $\set{A, D, \alpha, \gamma, \zeta}$, $\set{B, D, \beta, \delta, \zeta}$.
It is easy to see that any \dungcomplete{} extension can be extended to a corresponding complete extension 
adding the attacks that arise from it (through the \opexttoafra{} operator) and, conversely, that any complete extension corresponds to a \dungcomplete{} extension if we consider only the arguments included in it.
\end{example}

\begin{figure}[ht]
	\centering
	\includegraphics[scale=0.25]{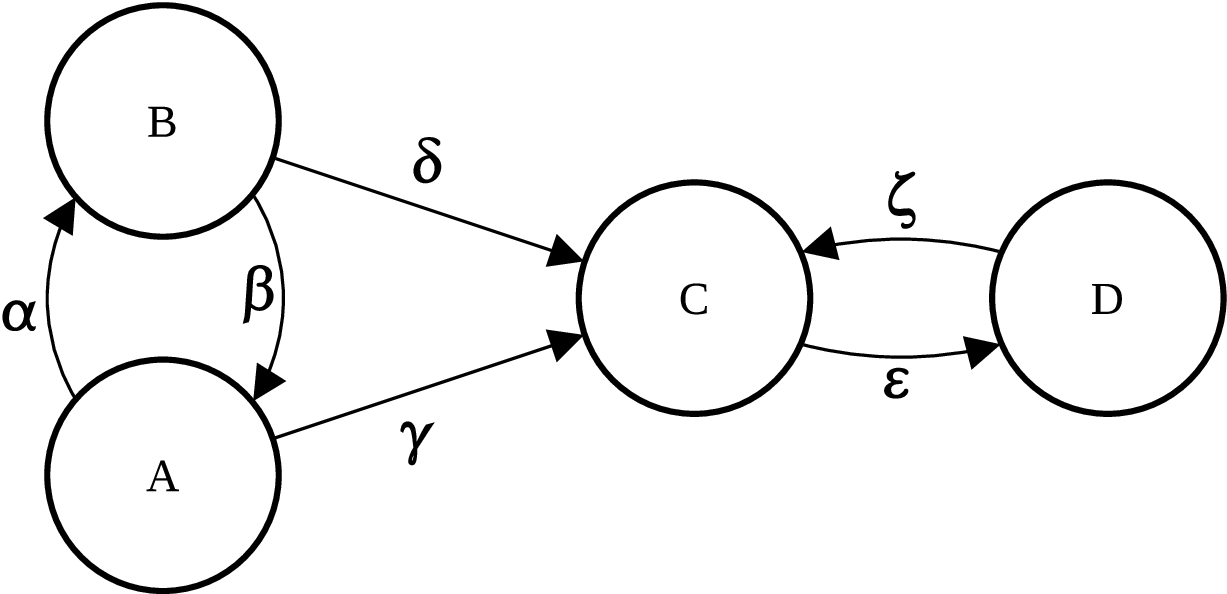}
	\caption{Graphical representation of Example \ref{ex_afraasaf}.}
	\label{fig:afraasaf}
\end{figure}

This result can be extended to prove an analogous correspondence
between the preferred and the \dungpreferred{} extensions, the grounded and the \dunggrounded{} extension,
as well as the stable and the \dungstable{} extensions.

\begin{proposition}
\label{prop_corr_pref_complete}
  Let $\anAFRA = \tup{\A}{\R}$ be an \AF. 
  Then, $\SUB$ is a preferred extension of $\anAFRA$ \ifff{} 
  $\SUB = \exttoafra{\U}$ where $\U$ is a \dungpreferred{} extension of $\anAFRA$.
\begin{proof}
  $\Rightarrow$. If $\SUB$ is a preferred extension, then by Lemma \ref{lemma_preferred_as_complete}
  it is also a complete extension, therefore by Proposition \ref{prop_corr_af_complete}
  $\SUB = \exttoafra{\U}$ where $\U$ is a \dungcomplete{} extension of $\anAFRA$.
  Assume by contradiction that $\U$ is not a \dungpreferred{} extension of $\anAFRA$:
  then there is a \dungpreferred{} extension (which is also a complete extension) 
  $\U' \subseteq \A$ such that $\U \subsetneq \U'$, 
  entailing by Lemma \ref{lemma_afra_operator_properties}(2) that $\exttoafra{\U} \subsetneq \exttoafra{\U'}$.
  Furthermore, Proposition \ref{prop_corr_af_complete} entails that
  $\exttoafra{\U'}$ is a complete extension, and in particular an admissible set:
  but this contradicts the fact that $\exttoafra{\U} = \SUB$ is a preferred extension.
  
  $\Leftarrow$. If $\U$ is a \dungpreferred{} extension of $\anAFRA$ it is in particular a \dungcomplete{} extension,
  therefore Proposition \ref{prop_corr_af_complete} entails that $\SUB = \exttoafra{\U}$ 
  is a complete extension of $\anAFRA$.
  Assuming by contradiction that it is not maximal, 
  by Theorem \ref{thm_adm_in_pref} and Lemma \ref{lemma_preferred_as_complete}
  there is a complete extension $\SUB'$ such that $\SUB \subsetneq \SUB'$,
  and by Proposition \ref{prop_corr_af_complete} it turns out that 
  $\SUB' = \exttoafra{\U'}$ where $\U'$ is a \dungcomplete{} extension of $\anAFRA$.
  However, $\SUB \subsetneq \SUB'$ entails by Lemma \ref{lemma_afra_operator_properties}(2) that
  $\U \subsetneq \U'$, contradicting the fact that $\U$ is a \dungpreferred{} extension of $\anAFRA$.
\end{proof}
\end{proposition}

\begin{proposition}
  Let $\anAFRA = \tup{\A}{\R}$ be an \AF{} and let $\SUB$ be the grounded extension of $\anAFRA$. 
  Then, $\SUB = \exttoafra{\U}$ where $\U$ is the \dunggrounded{} extension of $\anAFRA$.
\begin{proof}
  Since $\SUB$ is by definition a complete extension, by Proposition \ref{prop_corr_af_complete}
  it is the case that $\SUB = \exttoafra{\U}$ with $\U$ a \dungcomplete{} extension of $\anAFRA$.
  Assume by contradiction that $\U$ is not the least \dungcomplete{} extension.
  Then, letting $\U'$ the \dunggrounded{} extension of $\anAFRA$, we have
  $\U' \subsetneq \U$ which by Lemma \ref{lemma_afra_operator_properties}(2) entails that
  $\exttoafra{\U'} \subsetneq \exttoafra{\U}$, where $\exttoafra{\U'}$ is by Proposition \ref{prop_corr_af_complete}
  a complete extension of $\anAFRA$: but this contradicts the fact that $\SUB = \exttoafra{\U}$ is the least complete extension of $\anAFRA$.
\end{proof} 
\end{proposition}

\begin{proposition}
  Let $\anAFRA = \tup{\A}{\R}$ be an \AF. 
  Then, $\SUB$ is a stable extension of $\anAFRA$ \ifff{} 
  $\SUB = \exttoafra{\U}$ where $\U$ is a \dungstable{} extension of $\anAFRA$.
\begin{proof}
  $\Rightarrow$. If $\SUB$ is a stable extension, then by Lemma \ref{lemma_afra_stable_preferred}
  it is also a preferred and thus a complete extension, 
  therefore by Proposition \ref{prop_corr_af_complete} it is the case that
  $\SUB = \exttoafra{\U}$ where $\U$ is a \dungcomplete{} extension of $\anAFRA$ 
  and in particular a conflict-free set.
  To see that $\U$ is a \dungstable{} extension of $\anAFRA$,
  consider a generic $\arga \in \A, \arga \notin \U$.
  By definition of $\exttoafra{\U}$ it must be the case that $\arga \notin \exttoafra{\U}$,
  thus since $\SUB$ is a stable extension 
  $\exists \attackalpha \in \exttoafra{\U}$ such that $\AFRattacks{\attackalpha}{\arga}$.
  The conclusion follows from the fact that, by definition of $\exttoafra{\U}$, $\src{\attackalpha} \in \U$, 
  and $\attacksAF{\src{\attackalpha}}{\arga}$.
   
  $\Leftarrow$. Since $\U$ is a \dungstable{} extension of $\anAFRA$ and thus a \dungcomplete{} extension,
  by Proposition \ref{prop_corr_af_complete} it holds that
  $\SUB = \exttoafra{\U}$ is a complete extension and thus a conflict-free set of $\anAFRA$.
  Let $\elementv$ be a generic element not belonging to $\SUB$.
  If $\elementv \in \A$ then by definition of $\exttoafra{\U}$ it is the case that 
  $\elementv \notin \U$, thus $\elementv \in \dungplus{\U}$ since $\U$ is a \dungstable{} extension.
  If $\elementv \in \R$ then by definition of $\exttoafra{\U}$ it holds that $\src{\elementv} \notin \U$, 
  thus again $\src{\elementv} \in \dungplus{\U}$.
  In any case, $\exists \attackalpha$ with $\src{\attackalpha} \in \U$ such that 
  $\AFRattacks{\attackalpha}{\elementv}$,
  and the conclusion follows from the fact that, by definition of $\exttoafra{\U}$, $\attackalpha \in \SUB$.
\end{proof}
\end{proposition}

The relationships among \dungpreferred{} and preferred extensions, \dungstable{} and stable extensions, and \dunggrounded{} and grounded extensions can be easily identified in Example \ref{ex_afraasaf}.

\begin{examplecont}[\ref{ex_afraasaf}]
$\set{A, D}$ and $\set{B, D}$ are both \dungpreferred{} and \dungstable{} extensions of $\athirdAFRA$, while $\set{A, D, \alpha, \gamma, \zeta}$ and $\set{B, D, \beta, \delta, \zeta}$ are, correspondingly, both stable and preferred extensions as defined in \AFRA.
The \dunggrounded{} extension of $\athirdAFRA$ is $\emptyset$ and coincides with the grounded extension as defined in \AFRA{} (note that $\exttoafra{\emptyset}=\emptyset$).
\end{examplecont}

The bijective correspondence also holds for semi-stable semantics.
To show this we have to first prove a property
concerning the relationship between the \opexttoafra{} operator and the range of a set.

\begin{lemma}
\label{lemma_afra_operator_property}
  Let $\anAFRA = \tup{\A}{\R}$ be an \AF{} and $\U \subseteq \A$ a set of arguments.
  Then, $\range{\exttoafra{\U}} = \exttoafra{(\drange{\U})}$.
\begin{proof}
  By definition
  $\range{\exttoafra{\U}} = 
   \exttoafra{\U} \cup \set{\elementv \in \A \cup \R 
           \mid \exists \attackalpha \in \exttoafra{\U} : \AFRattacks{\attackalpha}{\elementv}}$, \ie{}
  $\exttoafra{\U} \cup 
  \set{\arga \in \A \mid \exists \attackalpha \in \exttoafra{\U} : \AFRattacks{\attackalpha}{\arga}}  \cup
  \set{\attackbeta \in \R \mid \exists \attackalpha \in \exttoafra{\U} : \AFRattacks{\attackalpha}{\attackbeta}}$,
  which, taking into account that an attack $\attackalpha \in \exttoafra{\U}$ \ifff{} $\src{\attackalpha} \in \U$, 
  is in turn equal to
  $\exttoafra{\U} \cup \dungplus{\U} \cup 
  \set{\attackbeta \in \R \mid \exists \attackalpha \in \exttoafra{\U} : \AFRattacks{\attackalpha}{\attackbeta}}$.
  Since $\anAFRA$ is an \AF, $\AFRattacks{\attackalpha}{\attackbeta}$ with $\attackbeta \in \R$ 
  can only hold by indirect defeat,
  thus $\range{\exttoafra{\U}}$ can be expressed as 
  $\exttoafra{\U} \cup \dungplus{\U} \cup \set{\attackbeta \in \R \mid \src{\attackbeta} \in \dungplus{\U}}$,
  \ie{} $\exttoafra{\U} \cup \exttoafra{(\dungplus{\U})}$. 
  Now, by Lemma \ref{lemma_afra_operator_properties}(3) the last expression is equal to
  $\exttoafra{(\U \cup \dungplus{\U})}$, \ie{} $\exttoafra{(\drange{\U})}$.
\end{proof}
\end{lemma}

\begin{proposition}
  Let $\anAFRA = \tup{\A}{\R}$ be an \AF. 
  Then, $\SUB$ is a semi-stable extension of $\anAFRA$ \ifff{} 
  $\SUB = \exttoafra{\U}$ where $\U$ is a \dungsemistable{} extension of $\anAFRA$.
\begin{proof}
  $\Rightarrow$. Since $\SUB$ is a semi-stable extension of $\anAFRA$, 
  it is by definition also a complete extension, therefore
  by Proposition \ref{prop_corr_af_complete} it is the case that
  $\SUB = \exttoafra{\U}$ where $\U \subseteq \A$ is a \dungcomplete{} extension of $\anAFRA$.
  Assume by contradiction that $\U$ is not a \dungsemistable{} extension:
  then, there is a \dungcomplete{} extension $\U'$ of $\anAFRA$, $\U' \subseteq \A$,
  such that $\drange{\U} \subsetneq \drange{\U'}$, and letting $\SUB' \equiv \exttoafra{\U'}$
  yields $\SUB'$ a complete extension of $\anAFRA$ by Proposition \ref{prop_corr_af_complete}.
  However, $\drange{\U} \subsetneq \drange{\U'}$ entails by Lemma \ref{lemma_afra_operator_properties}(2)
  that $\exttoafra{(\drange{\U})} \subsetneq \exttoafra{(\drange{\U'})}$,
  which according to Lemma \ref{lemma_afra_operator_property} 
  is equivalent to $\range{\SUB} \subsetneq \range{\SUB'}$,
  contradicting the fact that $\SUB$ is a semi-stable extension of $\anAFRA$. 
  
  $\Leftarrow$. Since $\U$ is a \dungcomplete{} extension of $\anAFRA$,
  by Proposition \ref{prop_corr_af_complete} it holds that 
  $\SUB = \exttoafra{\U}$ is a complete extension of $\anAFRA$.
  Assume by contradiction that it is not semi-stable:
  then, there is a complete extension $\SUB' \subseteq (\A \cup \R)$ such that
  $\range{\SUB} \subsetneq \range{\SUB'}$, where according to Proposition \ref{prop_corr_af_complete}
  it is the case that $\SUB' = \exttoafra{\U'}$ with $\U' \subseteq \A$ a \dungcomplete{} extension of $\anAFRA$.
  However, applying Lemma \ref{lemma_afra_operator_property} to
  $\range{\SUB} \subsetneq \range{\SUB'}$ yields
  $\exttoafra{(\drange{\U})} \subsetneq \exttoafra{(\drange{\U'})}$, 
  which by Lemma \ref{lemma_afra_operator_properties}(2) holds \ifff{}
  $\drange{\U} \subsetneq \drange{\U'}$, contradicting the fact that 
  $\U$ is a \dungsemistable{} extension of $\anAFRA$.
\end{proof}
\end{proposition}

In order to exemplify the relationship between \dungsemistable{} and semi-stable extensions, let us consider again Example \ref{ex_afraasaf}.

\begin{examplecont}[\ref{ex_afraasaf}]
The  \dungsemistable{} extensions of $\athirdAFRA$ are 
$\set{A, D}$ and $\set{B, D}$, while $\set{A, D, \alpha, \gamma, \zeta}$ and $\set{B, D, \beta, \delta, \zeta}$ are, correspondingly, its semi-stable extensions.
\end{examplecont}

We finally provide the correspondence result for ideal semantics.

\begin{proposition}\label{prop_corr_af_ideal}
  Let $\anAFRA = \tup{\A}{\R}$ be an \AF{} and let $\SUB$ be the ideal extension of $\anAFRA$. 
  Then, $\SUB = \exttoafra{\U}$ where $\U$ is the \dungideal{} extension of $\anAFRA$.
\begin{proof}
  According to Lemma \ref{lemma_ideal_as_complete} $\SUB$ is a complete extension of $\anAFRA$,
  therefore, by Proposition \ref{prop_corr_af_complete},
  $\SUB = \exttoafra{\U}$ where $\U$ is a \dungcomplete{} extension of $\anAFRA$ (thus in particular \dungadmissible).
  By Proposition \ref{prop_corr_pref_complete},
  $\forall \ \E$ with $\E$ a \dungpreferred{} extension of $\anAFRA$ $\exttoafra{\E}$ is a preferred extension,
  thus by definition of ideal extension $\SUB = \exttoafra{\U} \subseteq \exttoafra{\E}$, 
  which by Lemma \ref{lemma_afra_operator_properties}(1) yields $\U \subseteq \E$. 
  To show that $\U$ is the \dungideal{} extension, we have to prove that $\U$ 
  is the maximal subset of $\A$ satisfying the latter condition.
  Assume by contradiction that this is not the case:
  then, there is a set $\U' \subseteq \A$ such that $\U \subsetneq \U'$ and 
  $\U'$ is a \dungcomplete{} extension contained in all the \dungpreferred{} extensions of $\anAFRA$.
  By Lemma \ref{lemma_afra_operator_properties}(2), 
  $\SUB = \exttoafra{\U} \subsetneq \exttoafra{\U'}$,
  where, according to Proposition \ref{prop_corr_af_complete}, 
  $\exttoafra{\U'}$ is a complete extension of $\anAFRA$, thus admissible.
  Moreover, by Proposition \ref{prop_corr_pref_complete} 
  for any preferred extension $\aPreferredExtension$ of $\anAFRA$
  $\aPreferredExtension = \exttoafra{\E}$ with $\E$ a \dungpreferred{} extension of $\anAFRA$,
  and since $\U' \subseteq \E$ according to Lemma \ref{lemma_afra_operator_properties}(1) we have that
  $\exttoafra{\U'} \subseteq \exttoafra{\E} = \aPreferredExtension$.
  Summing up, there is an admissible set, namely $\exttoafra{\U'}$,
  which is contained in all preferred extensions of $\anAFRA$ and such that $\SUB \subsetneq \exttoafra{\U'}$:
  but this contradicts the fact that $\SUB$ is the ideal extension of $\anAFRA$.
\end{proof}
\end{proposition}

\begin{examplecont}[\ref{ex_afraasaf}]
The  \dungideal{} extension of $\athirdAFRA$ is 
$\set{D}$, while $\set{D, \zeta}$ is, correspondingly, its ideal extension.
\end{examplecont}

\section{Expressing $\AFRA$ as an $\AF$}\label{SecGAFAF}

We consider now the issue of expressing an $\AFRA$ in terms of a traditional $\AF$ and drawing the relevant correspondences concerning the notions introduced in Sections \ref{sec_semantics_afra_1} and \ref{sec_semantics_afra_2}.
This kind of correspondence provides a very useful basis for further investigations as it allows one to reuse or adapt, in the context of $\AFRA$, the large corpus of theoretical results available in Dung's framework, in particular as far as computational complexity is concerned.

\begin{definition}
\label{def_from_afra_to_af}
Let $\anAFRA = \tup{\A}{\R}$ be an $\GAF$, the corresponding \AF{} $\toAF{\anAFRA} = \tup{\Atilde}{\Rtilde}$ is defined as follows:

\begin{itemize}
\item $\Atilde = \A \cup \R$;
\item $\Rtilde = \{(\elementv, \elementw) | \elementv, \elementw \in \A \cup \R$ and $\AFRattacks{\elementv}{\elementw}\}$.
\end{itemize}
\end{definition}

In words, both arguments and attacks of the original $\GAF$ $\anAFRA$ become arguments of its corresponding 
\AF-version $\toAF{\anAFRA}$, while the defeat relations in $\AF$ correspond to all direct and indirect defeats in the original $\GAF$.
We can now examine the relationships between the relevant notions in $\anAFRA$ and $\toAF{\anAFRA}$ showing that they are all bijections as desirable.

\begin{proposition}
Let $\anAFRA = \tup{\A}{\R}$ an $\GAF$ and $\toAF{\anAFRA} = \tup{\Atilde}{\Rtilde}$ its corresponding \AF, $\SUB \subseteq \A \cup \R$ and $\elementa \in \A \cup \R$:
\begin{enumerate}
\item $\SUB$ is a conflict--free set for $\anAFRA$ iff $\SUB$ is a \dungconffree{} set for $\toAF{\anAFRA}$;
\item $\elementa$ is acceptable w.r.t. $\SUB \subseteq \A \cup \R$ in $\anAFRA$ 
iff $\elementa$ is \dungacceptable{} w.r.t. $\SUB$ in $\toAF{\anAFRA}$;
\item $\SUB$ is an admissible set for $\anAFRA$ iff $\SUB$ is a \dungadmissible{} set for $\toAF{\anAFRA}$;
\item $\SUB$ is a preferred extension for $\anAFRA$ iff $\SUB$ is a \dungpreferred{} extension for $\toAF{\anAFRA}$;
\item $\SUB$ is a stable extension for $\anAFRA$ iff $\SUB$ is a \dungstable{} extension for $\toAF{\anAFRA}$;
\item $\SUB$ is a complete extension for $\anAFRA$ iff $\SUB$ is a \dungcomplete{} extension for $\toAF{\anAFRA}$;
\item $\SUB$ is the grounded extension for $\anAFRA$ iff $\SUB$ is the \dunggrounded{} extension for $\toAF{\anAFRA}$;
\item $\SUB$ is a semi-stable extension for $\anAFRA$ iff $\SUB$ is a \dungsemistable{} 
extension for $\toAF{\anAFRA}$;
\item $\SUB$ is the ideal extension for $\anAFRA$ iff $\SUB$ is the \dungideal{} extension for $\toAF{\anAFRA}$.
\end{enumerate}
\end{proposition}

\begin{proof}\mbox{}
\begin{enumerate}
\item 
\label{proof_gaf_to_af_conflict-free}
The conclusion follows directly from Definition \ref{def_from_afra_to_af}, taking into account Definitions \ref{defrecall} and \ref{conflict-free_gaf}.
\item 
\label{proof_gaf_to_af_acceptability}
$\Rightarrow$. Let $\elementa \in \A \cup \R$ be acceptable w.r.t. $\SUB \subseteq \A \cup \R$ in $\anAFRA$ and suppose $\elementa$ is not \dungacceptable{} w.r.t. $\SUB$ in $\toAF{\anAFRA}$. So, there exists $\elementb \in \Atilde = \A \cup \R$ s.t. $(\elementb, \elementa) \in \Rtilde$ and $\nexists \elementc \in \SUB$ s.t. $(\elementc, \elementb) \in \Rtilde$. From Definition \ref{def_from_afra_to_af}, $(\elementb, \elementa) \in \Rtilde$ iff $\AFRattacks{\elementb}{\elementa}$ and $(\elementc, \elementb) \in \Rtilde$ iff $\AFRattacks{\elementc}{\elementb}$. Then $\exists \elementb \in \A \cup \R$ s.t. $\AFRattacks{\elementb}{\elementa}$ and $\nexists \elementc \in \SUB$ s.t. $\AFRattacks{\elementc}{\elementb}$. Therefore $\elementa$ is not acceptable w.r.t. $\SUB$ in $\anAFRA$. Contradiction.

$\Leftarrow$.
Follows the same reasoning line with obvious modifications.

\item
%
Follows directly from 1 and 2.

\item 
It follows directly from 3 since both in \AFRA{} (Definition \ref{def_preferred_afra}) 
and in \AF{} (Definition \ref{def_sem_recall}) preferred (\dungpreferred) extensions are defined 
as maximal w.r.t. set inclusion admissible (\dungadmissible) sets.

\item 
From 1 conflict-free sets are in correspondence between $\anAFRA$ and $\toAF{\anAFRA}$. From Definition \ref{def_from_afra_to_af}, it is easy to see that if $\forall \elementa \in \A \cup \R, \elementa \notin \SUB$, $\exists \elementb \in \SUB$ s.t. $\AFRattacks{\elementb}{\elementa}$ (Definition \ref{def_stab_afra}) then also $\forall \arga \in \Atilde \setminus \SUB$, $\exists \argb \in \SUB$ s.t. $\attacksAF{\argb}{\arga}$ (Definition \ref{defrecall}) and viceversa.

\item 
From 2 it follows that the characteristic function $\charfunction{\anAFRA}$ of $\anAFRA$ is equal to the \dungcharacteristic{} function $\carfun{\toAF{\anAFRA}}$ of its corresponding \AF{} $\toAF{\anAFRA}$. Then the conclusion follows from 3, taking into account Definitions \ref{def_complete_afra} and \ref{def_sem_recall}.

\item 
It follows from 6 as the grounded extension is the least complete (\dungcomplete) extension both in \AFRA{} (Lemma \ref{lemma_grounded_as_complete}) and in \AF{} (Definition \ref{def_sem_recall}).

\item 
It follows from 6, taking into account Definitions \ref{def_semi_afra} and \ref{def_sem_recall} and noting that, in virtue of Definition \ref{def_from_afra_to_af}, the range of a set $\SUB$ in \AFRA{} (Definition \ref{def_range_afra}) is equal to the \dungrange{} of $\SUB$ in $\toAF{\anAFRA}$ (Definition \ref{defrecall}).

\item 
It follows directly from 3 and 4.
\qedhere
\end{enumerate}
\end{proof}
 
\section{Comparison with related works}\label{sec_discuss}

\subsection{The Extended Argumentation Framework}

In recent years a generalization of Dung's framework to encompass attacks to attacks has been proposed in \cite{modgil2007,modgil2009}, called Extended Argumentation Framework (\EAF{}).
This approach is motivated by the need to express preferences between arguments and supports a very interesting form of meta-level argumentation about the values that arguments promote.
In \EAF{} a limited notion of attacks to attacks is encompassed: only attacks whose target is an argument can be attacked, while attacks whose target is another attack can not be attacked in turn. In short, only one level of attacks to attacks is allowed.
Referring to Figure \ref{fig_esempio}, only the attack originated from $\argp$ can be represented, while the one originated from $\argn$ can not.

We recall briefly the main definitions of \EAF{} formalism.

\begin{definition}
\label{def_modgil_eaf}
An \emph{Extended Argumentation Framework} (\EAF) is a tuple \ModgilEAFTuple{} s.t. \ModgilArgs{} is a set of arguments, and:
\begin{itemize}
\item $\ModgilR \subseteq \ModgilArgs \times \ModgilArgs$;
\item $\ModgilD \subseteq \ModgilArgs \times \ModgilR$;
\item if $(\argx, (\argy, \argz))$, $(\argxprime, (\argz, \argy)) \in \ModgilD$ then $(\argx, \argxprime), (\argxprime, \argx) \in \ModgilR$.
\end{itemize}

\end{definition}

\begin{definition}
\label{def_modgil_definitions}
Let \ModgilEAFTuple{} be an \EAF{} and $\unsetModgil \subseteq \ModgilArgs$. Then 
\begin{itemize}
\item 

$\arga$ $\ModgilDefeatSWord$ $\argb$ iff $(\arga, \argb) \in \ModgilR$  and $\nexists \argc \in \unsetModgil$ s.t. $(\argc, (\arga, \argb)) \in \ModgilD$. 
We write $\ModgilDefeatS{\arga}{\argb}{\unsetModgil}$ to denote that $\arga$ $\ModgilDefeatSWord$ $\argb$, and $\ModgilNDefeatS{\arga}{\argb}{\unsetModgil}$ to denote that $\arga$ does not $\ModgilDefeatSWord$ $\argb$;

\item 
$\unsetModgil$ is conflict-free iff $\forall \arga, \argb \in \unsetModgil$: if $(\arga, \argb) \in \ModgilR$, then $(\argb, \arga) \notin \ModgilR$ and $\exists \argc \in \unsetModgil$ s.t. $(\argc, (\arga, \argb)) \in \ModgilD$;

\item 
$\ModgilRS = \set{\ModgilDefeatS{\argx_1}{\argy_1}{\unsetModgil}, \ldots, \ModgilDefeatS{\argx_n}{\argy_n}{\unsetModgil}}$ is a reinstatement set for $\ModgilDefeatS{\argc}{\argb}{\unsetModgil}$ iff:
	\begin{enumerate}
	\item $\ModgilDefeatS{\argc}{\argb}{\unsetModgil} \in \ModgilRS $,
	\item for $i = 1 \ldots n, \argx_i \in \unsetModgil $,
	\item $\forall \ModgilDefeatS{\argx}{\argy}{\unsetModgil} \in \ModgilRS $, $\forall \argyprime$ s.t. $(\argyprime,(\argx, \argy)) \in \ModgilD$, there is a $\ModgilDefeatS{\argxprime}{\argyprime}{\unsetModgil} \in \ModgilRS$.
	\end{enumerate}

\item 
$\arga \in \ModgilArgs$ is acceptable w.r.t. $\unsetModgil$, iff $\forall \argb$ s.t. $\ModgilDefeatS{\argb}{\arga}{\unsetModgil}$, there is a $\argc \in \unsetModgil$ s.t. $\ModgilDefeatS{\argc}{\argb}{\unsetModgil}$ and there is a reinstatement set for $\ModgilDefeatS{\argc}{\argb}{\unsetModgil}$.
\end{itemize}
\end{definition}

Semantics notions for \EAF{} paralleling Dung's ones are proposed in \cite{modgil2009}.

\begin{definition}
\label{def_modgil_semantics}
Let $\unsetModgil$ be a conflict-free subset of \ModgilArgs{} in \ModgilEAFTuple. Then:
\begin{itemize}
\item $\unsetModgil$ is an admissible extension iff every argument in $\unsetModgil$ is acceptable w.r.t. $\unsetModgil$.
\item $\unsetModgil$ is a preferred extension iff $\unsetModgil$ is a set inclusion maximal admissible extension.
\item $\unsetModgil$ is a complete extension iff each argument which is acceptable w.r.t. $\unsetModgil$ is in $\unsetModgil$.
\item $\unsetModgil$ is a stable extension iff $\forall \argb \notin \unsetModgil, \exists \arga \in \unsetModgil$ s.t. $\arga$ $\ModgilDefeatSWord$ $\argb$.
\end{itemize}
\end{definition}

Differently from the case of \AF, in \EAF{} the characteristic function is defined on conflict-free sets only and is used to define  grounded semantics.

\begin{definition}
Let $\aModgilEAF = \ModgilEAFTuple$ be an \EAF, $\unsetModgil \subseteq \ModgilArgs$, and $2^{\ModgilArgsCF}$ denote the set of all conflict-free subset of \ModgilArgs. The characteristic function $\ModgilCharFunct: 2^{\ModgilArgsCF} \mapsto 2^{\ModgilArgs}$ is defined as $\ModgilCharFunct(\unsetModgil) = \set{\arga | \arga$ is acceptable w.r.t. $\unsetModgil}$.
\end{definition}

In \EAF{} the grounded semantics is defined only for finitary \EAF s.

\begin{definition}
\ModgilEAFTuple{} is finitary iff $\forall \arga \in \ModgilArgs$, the set $\set{\argb | (\argb, \arga) \in \ModgilR}$ is finite, and $\forall (\arga, \argb) \in \ModgilR$, the set $\set{\argc | (\argc, (\arga, \argb)) \in \ModgilD}$ is finite.
\end{definition}

\begin{definition}\label{def_eaf_grounded}
Let \aModgilEAF{} be a finitary \EAF{} and $\ModgilCharFunct^0 = \emptyset$, $\ModgilCharFunct^{i+1} = \ModgilCharFunct(\ModgilCharFunct^i)$. Then $\bigcup_{i = 0}^{\infty} (\ModgilCharFunct^i)$ is the grounded extension of \aModgilEAF. 
\end{definition}

It is possible to draw a direct correspondence from \EAF{} to \AFRA.

\begin{definition}[\AFRA-\EAF{} correspondence]
\label{def_afra_eaf}
For any \EAF{} $\aModgilEAF = \ModgilEAFTuple$ we define the corresponding \AFRA{} $\toAFRA{\aModgilEAF} = \tup{\ModgilArgs}{\ModgilR \cup \ModgilD}$.
\end{definition}

Apart from this formal correspondence at the definition level, four main points are worth remarking to compare \EAF{} and \AFRA. 

First, as already remarked, \EAF{} encompasses only attacks to attacks between arguments rather than the general issue of making any attack defeasible. In \cite{baronietal2009} an extension of \EAF{} (called \EAFplus) has been devised which allows for recursive attacks, while attempting to follow as close as possible the original \EAF{} definitions: it was shown that the \AFRA{} formalism is able to cover also \EAFplus.

A second issue concerns the notion of conflict-free set given in Definition \ref{def_modgil_definitions} and a constraint on the attack relation in Definition \ref{def_modgil_eaf}. Consider the following simple example.

\begin{example}
Let $\aModgilEAF = \ModgilEAFTuple$ s.t. $\ModgilArgs = \set{\arga, \argb, \argc}$, $\ModgilR = \set{(\arga, \argb)}$, $\ModgilD = \set{(\argc, (\arga, \argb))}$. Then, $\set{\arga, \argb, \argc}$ is a conflict-free set.
\end{example}

Let us add now the relation $(\argb, \arga)$ in \ModgilR.

\begin{example}
Let $\aModgilEAF = \ModgilEAFTuple$ s.t. $\ModgilArgs = \set{\arga, \argb, \argc}$, $\ModgilR = \set{(\arga, \argb), (\argb, \arga)}$, $\ModgilD = \set{(\argc, (\arga, \argb))}$. Then, $\set{\arga, \argb, \argc}$ is not a conflict-free set.
\end{example}

Suppose then that there is an argument $\argc'$ which attacks the attack $(\argb, \arga)$ but it is not the case that $\argc$ and $\argc'$ attack each other (this situation is illustrated in Figure \ref{fig_es_fake_eaf}).
Note that the third requirement of Definition \ref{def_modgil_eaf} is violated, therefore this case is not compatible with the \EAF{} definition. While this restriction is justified in the context of preference modelling, where \EAF{} has been conceived, it may turn out a limitation in other areas. Note anyway that even if this constraint was relaxed, by Definition \ref{def_modgil_definitions} the set $\unsetModgil=\set{\arga, \argb, \argc, \argc'}$ would not be conflict-free due to the mutual attack between $\arga$ and $\argb$.
However, coherently with other situations, $\unsetModgil$ should be conflict-free since no argument in $\unsetModgil$ $\ModgilDefeatSWord$ another element in $\unsetModgil$.

The situation shown in Figure \ref{fig_es_fake_eaf} can be directly encompassed in \AFRA{} where, in particular, the set $\set{\arga, \argb, \argc, \argc', \attackgamma, \attackdelta}$ is conflict-free.

\begin{figure}[ht]
	\centering
	\includegraphics[scale=0.25]{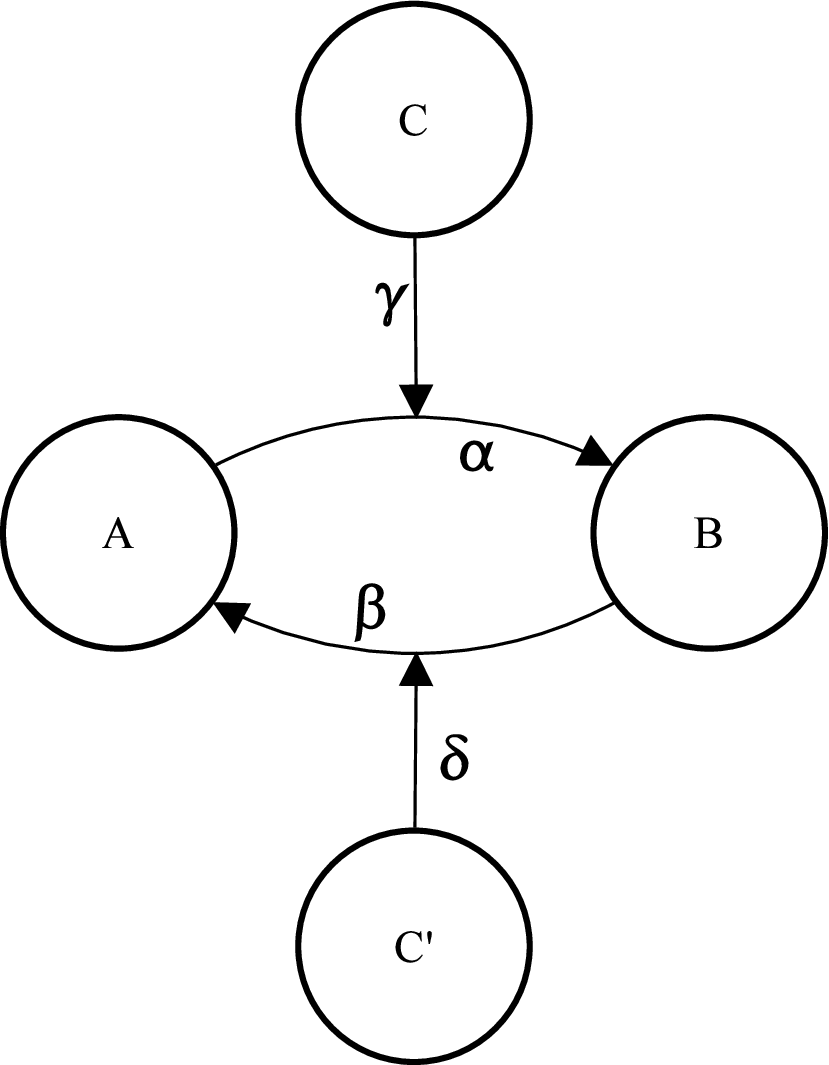}
	\caption{A situation forbidden in \EAF.}
	\label{fig_es_fake_eaf}
\end{figure}

The third issue concerns the fact that the \EAF{} characteristic function is not monotonic in general, while it is monotonic in the special cases\footnote{The notions of \HEAF{} and \psEAF{} and the relevant results are provided in \cite{modgil2009}.} of Hierarchical Extended Argumentation Framework (\HEAF), and Preference Symmetric Extended Argumentation Framework (\psEAF). In the \HEAF{} the sets \ModgilArgs{} and \ModgilR{} are partitioned in different levels which are ordered: any attack to attack can only start from a higher level than the one which contains the attacked attack. In the \psEAF{} only attacks between symmetrically attacking arguments can be attacked by arguments expressing preferences. 

These limitations do not apply to \AFRA{} where (as in \AF{}) the characteristic function is monotonic (Proposition \ref{prop_monotonic}), ensuring further desirable properties at the semantics level.

The fourth point regards the fact that, as remarked in \cite{modgil2009}, in general it does not hold that the grounded extension is the least complete extension: as a consequence, it is not guaranteed to be included in any complete extension nor, in particular, in any preferred extension.

\begin{example}
\label{es_problem_grounded_EAF}
(From \cite{modgil2009}). Consider $\aModgilEAF = \ModgilEAFTuple$ s.t. $\ModgilArgs = \set{\arga, \argb, \argc}$, $\ModgilR = \set{(\argb, \arga), (\argc, \argb)}$, $\ModgilD = \set{(\argb, (\argc, \argb))}$. The \virg{self-reinstating argument} \argb{} is overlooked by Definition \ref{def_eaf_grounded}: $\ModgilCharFunct^1 = \set{\argc}$, $\ModgilCharFunct^2 = \set{\argc, \arga}$, $\ModgilCharFunct^3 = \set{\argc, \arga}$ leading to the inclusion of $\arga$ in the grounded extension $\set{\argc, \arga}$. Definition \ref{def_modgil_semantics} gives rise to a different scenario where $\set{\argc}$, $\set{\argc, \arga}$ and $\set{\argc, \argb}$ are admissible, and $\set{\argc, \arga}$, $\set{\argc, \argb}$ are the preferred extensions. Hence, only $\argc$ is included in all preferred estensions while $\arga$ is not. 
\end{example}

\begin{figure}[ht]
	\centering
	\includegraphics[scale=0.25]{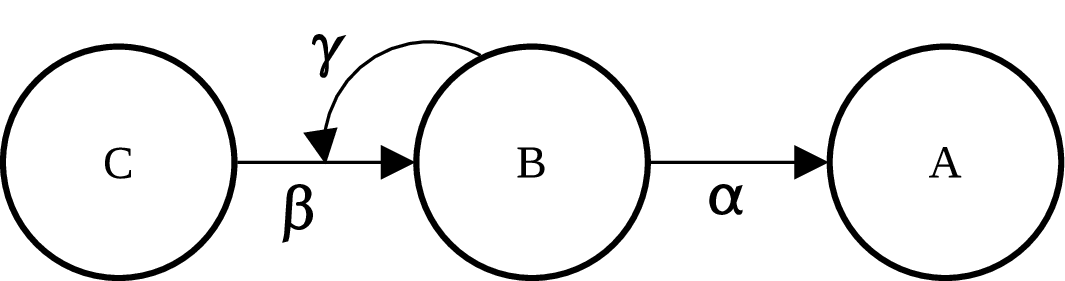}
	\caption{Example about the relation between grounded and preferred semantics in \EAF.}
	\label{fig_es_problem_grounded_EAF}
\end{figure}

Consider instead the example of Figure \ref{fig_es_problem_grounded_EAF} as formalized in \AFRA, \ie{} $\anAFRA = \tup{\A}{\R}$ s.t. $\A = \set{\arga, \argb, \argc}$ and $\R = \set{\attackalpha, \attackbeta, \attackgamma}$, where $\attackalpha = (\argb, \arga)$, $\attackbeta = (\argc, \argb)$, and $\attackgamma = (\argb, (\argc, \argb))$. There are two preferred extensions, namely $P_1(\anAFRA) = \set{\argc, \argb, \attackgamma, \attackalpha}$ and $P_2(\anAFRA) = \set{\argc, \arga, \attackbeta}$, while the  grounded extension is $\set{\argc}$. 

Therefore, the grounded semantics of \EAF{} is not in correspondence with the grounded semantics of \AFRA, and a \virg{classical} relation between semantics notions does not hold in \EAF.

\subsection{Higher Order Argumentation Framework}

In the context of reasoning about coalitions, a formalism encompassing attacks to attacks called Higher Order Argumentation Framework (HOAF) has been considered \cite{boellaetal2008,boellaetal2008b}. 

We recall literally the rather articulated definition of HOAF from \cite{boellaetal2008b}.

\begin{definition}
\label{def_high}
A Higher Order Argumentation Framework (HOAF) is a tuple $\BvVtuple$ where $\BvVArgsC$ is a set of coalition arguments, $\BvVArgsNot$ is a set of arguments such that $|\BvVArgsNot|$ = $|\BvVArgsC|$, $\BvVnot$ is a bijection from $\BvVArgsC$ to $\BvVArgsNot$, $\BvVArgsDiesis$ is a set of arguments that coalitions attacks attack each other, and $\BvVDiesis \subseteq (\BvVArgsC \times \BvVArgsNot) \cup (\BvVArgsNot \times \BvVArgsDiesis) \cup (\BvVArgsDiesis \times \BvVArgsC) \cup (\BvVArgsDiesis \times \BvVArgsDiesis)$ is a binary relation on the set of arguments such that for $\BvVarga \in \BvVArgsC$ and $\BvVargb \in \BvVArgsNot$ we have $\BvVarga \BvVDiesis \BvVargb$  if and only if $\BvVargb = \BvVnot(\BvVarga)$, and for each $\BvVarga \in \BvVArgsDiesis$, there is precisely one $\BvVargb \in \BvVArgsNot$ such that $\BvVargb \BvVDiesis \BvVarga$ and precisely one $\BvVargc \in \BvVArgsC \cup \BvVArgsDiesis$ such that $\BvVarga \BvVDiesis \BvVargc$. A higher order argumentation framework $\BvVtuple$ represents $\tup{\setarg}{\attackrel}$ if and only if $\setarg = \BvVArgsC \cup \BvVArgsNot \cup \BvVArgsDiesis$. The extensions of $\BvVtuple$ are the extensions of the represented argumentation framework.
\end{definition}

In HOAF attacks to attacks are not explicitly represented.
In fact, the attack relation $\BvVDiesis$ is not recursive and relates the elements of three distinct sets, namely $\BvVArgsC$, $\BvVArgsNot$, and $\BvVArgsDiesis$. The intuitive meaning of these three sets can be appreciated by regarding a HOAF as the result of a translation of an argumentation framework with recursive attacks into an \AF\footnote{Note that this kind of representation was also introduced in \cite{benchcapon&modgil2008} as an approach to rewrite an \EAF{} as a traditional \AF.}, where $\BvVArgsC$ represents the set of arguments and $\BvVArgsDiesis$ the set of attacks.
In fact for each element $\arga$ of $\BvVArgsC$ a corresponding element $\BvVnot(\arga)$ representing \virg{non acceptance} of $\arga$ is present in $\BvVArgsNot$. It is assumed that $\arga$ attacks $\BvVnot(\arga)$, i.e. $(\arga, \BvVnot(\arga)) \in \BvVDiesis$, but not viceversa.
Elements of $\BvVArgsDiesis$ represent attacks of the original argumentation framework with recursive attacks: each of these attacks has a source argument belonging to $\BvVArgsC$ and a target which is either an argument in $\BvVArgsC$ or another attack in $\BvVArgsDiesis$.
Then:
\begin{itemize}
\item for each element $\arga$ of $\BvVArgsC$, it is assumed that $\BvVnot(\arga)$ attacks in HOAF (through $\BvVDiesis$) all elements of $\BvVArgsDiesis$ whose source is $\arga$;
\item each element of $\BvVArgsDiesis$ attacks in HOAF (through $\BvVDiesis$) its target, either belonging to $\BvVArgsC$ (if it is an argument) or to $\BvVArgsDiesis$ (if it is another attack).
\end{itemize}

Finally, by simply treating $\BvVArgsC \cup \BvVArgsNot \cup \BvVArgsDiesis$ as an undistinguished set of arguments, ignoring the $\BvVnot$ relation and using $\BvVDiesis$ as an attack relation, a HOAF can be regarded as a traditional \AF{} and the relevant semantics notions are defined.

It emerges from the above discussion that, given a HOAF, a corresponding \AFRA{} can be defined by considering only elements of $\BvVArgsC$ as arguments and applying a sort of inversion of the implicit translation procedure, as proposed in Definition \ref{def_afra_hoaf}.

\begin{definition} \label{def_afra_hoaf}
For any $\aBvV = \BvVtuple$ we define the corresponding \AFRA{} $\toAFRA{\aBvV} = \tup{\BvVArgsC}{\R}$ s.t. $(\BvVarga, \elementv) \in \R$ if and only if 
$(\BvVarga, \elementv)$ represents $(\BvVargb, \BvVargc)\in \BvVDiesis$.
$(\BvVarga, \elementv) \in \R$ is said to represent $(\BvVargb, \BvVargc)\in \BvVDiesis$ if either of the following conditions holds:
\begin{itemize}
\item $\BvVargb \in \BvVArgsDiesis$, $(\BvVnot(\arga), \BvVargb) \in \BvVDiesis$, $\BvVargc \in \BvVArgsC$ and $\elementv = \BvVargc$;
\item $\BvVargb \in \BvVArgsDiesis$, $(\BvVnot(\arga), \BvVargb) \in \BvVDiesis$, $\BvVargc \in \BvVArgsDiesis$ and $\exists (\BvVargc, \BvVargd) \in \BvVDiesis$ such that $\elementv$ represents $(\BvVargc, \BvVargd)$.
\end{itemize}
\end{definition}

From Definition \ref{def_afra_hoaf} it emerges that any HOAF can be reduced to an \AFRA{} while the (implicit) translation algorithm previously described can be used to translate an \AFRA{} into a HOAF. Hence the two formalisms feature the same expressiveness, \AFRA{} providing however a simpler and cleaner representation and not requiring in particular the use of the additional \virg{not} arguments. Semantics notions are directly introduced in the \AFRA{} formalism and shown to satify desirable properties in relation to the representation of attacks to attacks while semantics notions in HOAF are indirectly introduced with reference to a \virg{represented \AF}, and are not accompanied by such an analysis nor by the statement of reference requirements.

To exemplify, let us consider the Bob's last minute dilemma formalised by an HOAF and shown in Figure \ref{fig_esempio_hoaf}.

\begin{example}
Let $\BvVtuple$ be an Higher Order Argumentation Framework, where $\BvVArgsC = \set{\arga,$ $\argn,$ $\argp,$ $\argc,$ $\argg}$, $\BvVArgsNot = \set{\BvVnot(\arga),$ $\BvVnot(\argn),$ $\BvVnot(\argp),$ $\BvVnot(\argc),$ $\BvVnot(\argg)}$, 
$\BvVArgsDiesis = \set{\argc-\argg,$ $\argg-\argc,$ $\argp-(\argc-\argg)$, $\argn-(\argp-(\argc-\argg)),$ $\arga-\argn}$ and 
$\BvVDiesis = \set{(\arga, \BvVnot(\arga)),$ $(\BvVnot(\arga), \arga-\argn),$ $(\arga-\argn, \argn),$ $(\argn, \BvVnot(\argn)),$ $(\BvVnot(\argn), \argn-(\argp-(\argc-\argg))),$ $(\argn-(\argp-(\argc-\argg)), \argp-(\argc-\argg))$, $(\argp, \BvVnot(\argp)),$ $(\BvVnot(\argp), \argp-(\argc-\argg)),$ $(\argp-(\argc-\argg), \argc-\argg),$ $(\argc, \BvVnot(\argc)),$ $(\BvVnot(\argc), \argc-\argg),$ $(\argc-\argg, \argg),$ $(\argg, \BvVnot(\argg)),$ $(\BvVnot(\argg), \argg-\argc),$ $(\argg-\argc, \argc)}$.

The only complete, grounded, preferred, stable, semi-stable, ideal extension, according with Definition \ref{def_high} is $\set{\arga,$ $\argp,$ $\argg,$ $\BvVnot(\argn),$ $\BvVnot(\argc),$ $\arga-\argn,$ $\argp-(\argc-\argg),$ $\argg-\argc}$.
\end{example}

\begin{figure}[ht]
	\centering
	\includegraphics[scale=0.25]{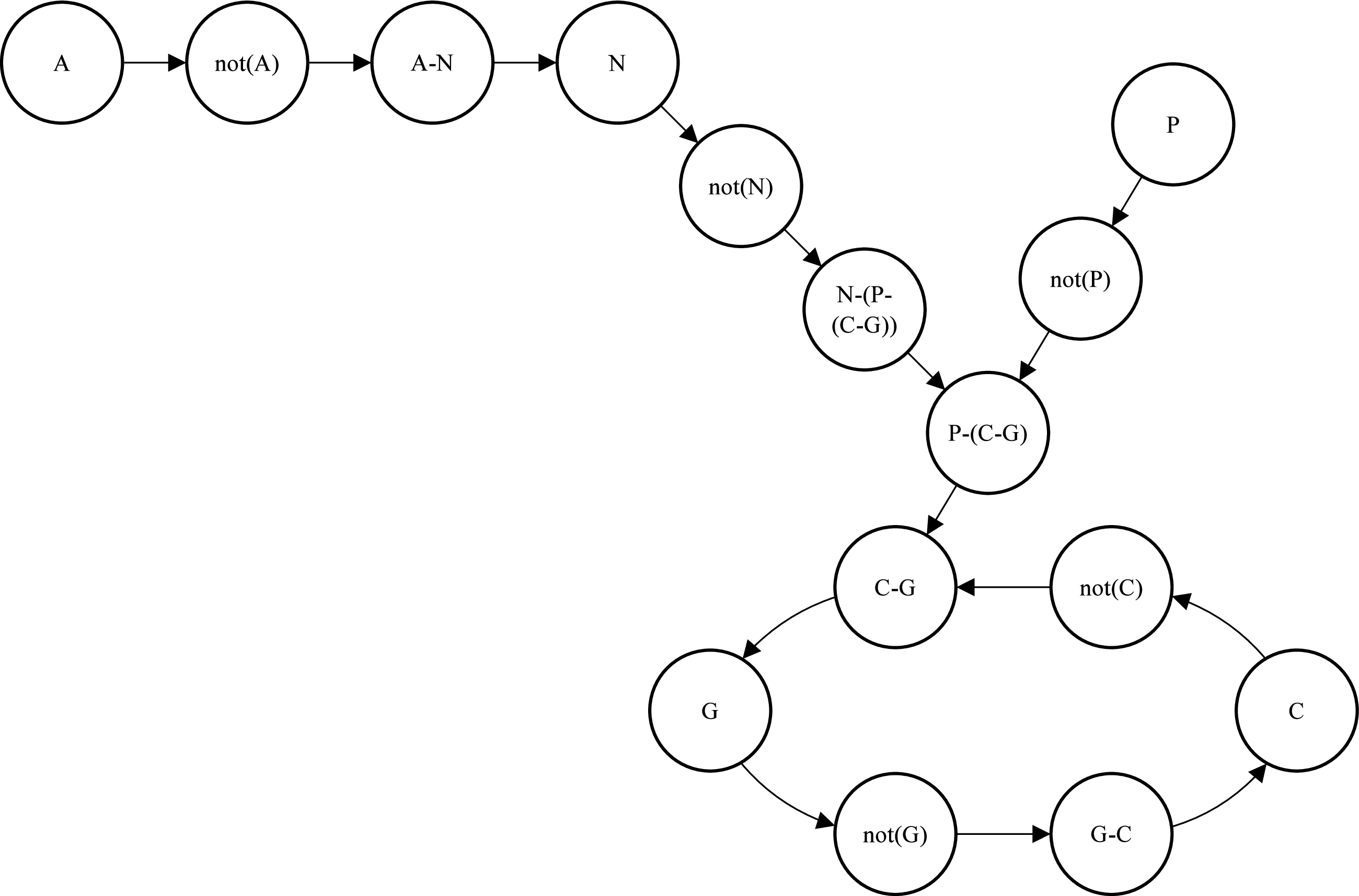}
	\caption{Bob's last minute dilemma formalised by a HOAF.}
	\label{fig_esempio_hoaf}
\end{figure}

Further investigation of the relations between HOAF and \AFRA{} is an interesting direction of future work.

\section{Conclusion and future works}\label{sec_concl}

\AFRA{} is a novel abstract argumentation formalism encompassing unlimited attacks to attacks in a quite simple formal setting and satisfying a set of basic requirements mainly concerning relationships with Dung's original \AF.
While direct correspondences between \AFRA{} and \AF{} are achieved at several levels, as shown in the paper, we remark that \AFRA{} represents a significant conceptual advancement by regarding attacks as defeasible entities themselves. This enables the representation of situations where some kind of reasoning about attacks is carried out, whose usefulness is suggested by several recent works in the literature. In particular, meta-argumentation \cite{wooldridgeetal2005,boellaetal2009} represents a promising research area where defeasibility of attacks may play a significant role as a useful modelling tool. 
Currently a full-fledged formalization of reasoning contexts where recursive attacks are required can be regarded as an important future research task: apart the issue of meta-argumentation mentioned above, a detailed analysis has been carried out up to now only for the case of reasoning about preferences in \cite{modgil2009} where just one level of recursion is considered.

Independently of the actual representation needs of specific reasoning contexts, it can be remarked that \AFRA{}, while pursuing a high generality in the definition of recursive attacks, achieves at the same time the goal of providing a simpler formalism than other (even less expressive) proposals.
In fact other approaches to represent
attacks to attacks in the literature appear to resort to somewhat more complicated formal structures and their semantics properties are still to be analyzed in detail and/or reveal a looser correspondence with Dung's ones, as discussed in Section \ref{sec_discuss}.
\AFRA{} appears anyway to be able to encompass these formalisms while, by explicit design choice, it does not cover more radical departures from Dung framework, where the notion of recursive attacks is combined with those of argument strength or of joint and disjunctive attack \cite{barringeretal2005,gabbay2009,gabbay2009b}. Analyzing their relationships with \AFRA{} represents an interesting direction of future work.

As to actual implementation in software tools, \AFRA{} has recently been included in ASPARTIX\footnote{We are grateful to the ASPARTIX team for having included \AFRA{} in their tool and having provided us useful information and support on this matter.}, a software tool for implementing argumentation frameworks using Answer Set Programming. As stated in \cite{eglyetal2008,eglyetal2008bis} this approach uses a fixed logic program which is capable of computing the different forms of extension from a given framework which is given as input. Due to this simple architecture, the system is easily extensible and suitable for rapid prototyping.
This shows that, by its simplicity, the \AFRA{} formalism lends itself to rather straightforward implementation. As to computational complexity, the results in Sections \ref{sec_compatibility_afra_af} and \ref{SecGAFAF} suggest that the many already known results for \AF{} can be applied to \AFRA{} too.

On the application side, we have mentioned in the paper several contexts where recursive attacks can be useful. The area of modelling articulated decision processes involving reasoning with values is particularly interesting as it has been the subject of detailed analysis leading to the proposal of the Value-Based Argumentation Framework (\VAF) \cite{benchcapon2003,benchcaponetal2007}. From this perspective an analysis of the relationships between \AFRA{} and \VAF{} provides a significant direction of future work: some preliminary results are provided in \cite{baronietal2009b}.

\bibliographystyle{elsarticle-num}
\bibliography{refs}

\end{document}